%% file: neurips_2023.tex
\documentclass{article}

\PassOptionsToPackage{square,numbers,sort&compress}{natbib}
\usepackage[final]{neurips_2023}

\usepackage[utf8]{inputenc} 
\usepackage[T1]{fontenc}    
\usepackage{xcolor}         

\definecolor{c1}{HTML}{2F70AF} 
\usepackage[colorlinks=true, citecolor=c1, linkcolor=c1]{hyperref} 
\usepackage{bm}
\usepackage{enumitem}
\usepackage{amsmath}        
\usepackage{MnSymbol}        
\usepackage{amsthm}
\usepackage{bbding}
\usepackage{url}            
\usepackage{booktabs}       
\usepackage{amsfonts}       
\usepackage{nicefrac}       
\usepackage{microtype}      
\usepackage{graphicx}
\usepackage{subfigure}
\usepackage{wrapfig}

\usepackage{chngcntr}       
\usepackage{multirow}

\usepackage{algorithm}
\usepackage[noend]{algpseudocode}
\newtheorem{thm}{Theorem}[section]

\newtheorem{definition}{Definition}[section]
\newtheorem{lem}{Lemma}[section]
\newtheorem{cor}{Corollary}[section]
\newtheorem{asmp}{Assumption}[section]
\newcommand{\ie}{\textit{i}.\textit{e}., }
\newcommand{\eg}{\textit{e}.\textit{g}., }
\newcommand{\wrt}{\textit{w}.\textit{r}.\textit{t}. }
\newcommand{\E}{\mathbb{E}}
\newcommand{\p}{\mathbb{P}}

\newcommand{\pt}{\mathbb{P}^\mathrm{T=1}}
\newcommand{\pc}{\mathbb{P}^\mathrm{T=0}}
\newcommand{\cY}{\mathcal{Y}}
\newcommand{\bpi}{\pmb{\pi}}
\newcommand{\loss}{l_{\psi,\phi}}
\newcommand{\epehe}{\epsilon_{\mathrm{PEHE}}}

\newcommand{\sig}{\rlap{$^*$}}
\newcommand{\indep}{\perp \!\!\! \perp}

\graphicspath{ {./images/} }

\usepackage{pgfplots}
\pgfplotsset{compat=1.15}
\usepackage{mathrsfs}
\usetikzlibrary{arrows}
\title{Optimal Transport for Treatment Effect Estimation}

\author{
  Hao Wang$^{1}$\quad Zhichao Chen$^{1}$ \quad Jiajun Fan$^{2}$ \quad Haoxuan Li$^{3}$ \quad Tianqiao Liu$^{4}$ \\ 
  \textbf{Weiming Liu$^{1}$ \quad Quanyu Dai$^{5}$\thanks{Corresponding author.} \quad Yichao Wang$^{5}$ \quad Zhenhua Dong$^{5}$ \quad Ruiming Tang$^{5}$}\\
  $^1$Zhejiang University \quad
  $^2$Tsinghua University \quad
  $^3$Peking University\\
  $^4$Purdue University \quad $^5$ Huawei Noah's Ark Lab\\
  \texttt{haohaow@zju.edu.cn \quad quanyu.dai@connect.polyu.hk}
}

\begin{document}
\definecolor{ffvvqq}{rgb}{1,0.3333333333333333,0}
\definecolor{qqqqff}{rgb}{0,0,1}
\definecolor{ccqqqq}{rgb}{0.8,0,0}
\definecolor{qqwuqq}{rgb}{0,0.39215686274509803,0}

\maketitle

\begin{abstract}
  Estimating conditional average treatment effect from observational data is highly challenging due to the existence of treatment selection bias. Prevalent methods mitigate this issue by aligning distributions of different treatment groups in the latent space. 
  However, there are two critical problems that these methods fail to address: (1) mini-batch sampling effects (MSE), which causes misalignment in non-ideal mini-batches with outcome imbalance and outliers; (2) unobserved confounder effects (UCE), which results in inaccurate discrepancy calculation due to the neglect of unobserved confounders. To tackle these problems, we propose a principled approach named \textbf{E}ntire \textbf{S}pace \textbf{C}ounter\textbf{F}actual \textbf{R}egression (ESCFR), which is a new take on optimal transport in the context of causality. Specifically, based on the framework of stochastic optimal transport, we propose a relaxed mass-preserving regularizer to address the MSE issue and design a proximal factual outcome regularizer to handle the UCE issue. Extensive experiments demonstrate that our proposed ESCFR can successfully tackle the treatment selection bias and achieve significantly better performance than state-of-the-art methods.
\end{abstract}

\input{content}

\section{Conclusion}
Due to the effectiveness of mitigating treatment selection bias, representation learning has been the primary approach to estimating individual treatment effect. 
However, existing methods neglect the mini-batch sampling effects and unobserved confounders, which hinders them from handling the treatment selection bias. A principled approach named ESCFR is devised based on a generalized OT problem. Extensive experiments demonstrate that ESCFR can mitigate MSE and UCE issues, and achieve better performance compared with prevalent baseline models.

Looking ahead, two research avenues hold promise for further exploration. The first avenue explores the use of normalizing flows for representation mapping, which offers the benefit of invertibility~\citep{node} and thus aligns with the assumptions set forth by \citet{cfr}. The second avenue aims to apply our methodology to real-world industrial applications, such as bias mitigation in recommendation systems~\citep{escm}, an area currently dominated by high-variance reweighting methods.
\section*{Acknowledgements}
This work is supported by National Key R\&D Program of China (Grant No. 2021YFC2101100), National Natural Science Foundation of China (62073288, 12075212, 12105246, 11975207) and Zhejiang University NGICS Platform.
\newpage
\bibliography{ref,causal}
\bibliographystyle{plainnat}
\newpage
\setcounter{page}{1}
\appendix
\onecolumn
\input{appendix}

\end{document}

%% file: content.tex
\section{Introduction}\label{sec:introduction}

The estimation of the conditional average treatment effect (CATE) through randomized controlled trials serves as a cornerstone in causal inference, finding applications in diverse fields such as health care~\citep{drnet}, e-commerce~\citep{uplift,wu2022opportunity,DBLP:conf/kdd/LiZWKL023}, and education~\citep{cordero2018causal}.
Although randomized controlled trials are often considered the gold standard for CATE estimation~\citep{pearl2018book,li2023www}, the associated financial and ethical constraints can make them infeasible in practice~\cite{li2022multiple,wustable}.
Therefore, alternative approaches that rely on observational data have gained prominence. For instance, drug developers could utilize post-marketing monitoring reports to assess the efficacy of new medications rather than conducting costly clinical A/B trials.
With the growing access to observational data, estimating CATE from observational data has attracted intense research interest~\cite{wuite,yangmixed,DBLP:conf/icml/LiZCGLW23}.

Estimating CATE with observational data has two main challenges: (1) missing counterfactuals, \ie only one factual outcome out of all potential outcomes can be observed; (2) treatment selection bias, \ie individuals have different propensities for treatment, leading to non-random treatment assignments and a resulting covariate shift between the treated and untreated groups~\cite{wustable,wang2023out}. 
Traditional meta-learners~\cite{kunzel2019metalearners} handle the issue of missing counterfactuals by breaking down the CATE estimation into more tractable sub-problems of factual outcome estimation.
However, the covariate shift makes it difficult to generalize the factual outcome estimators trained within respective treatment groups to the entire population and thus biases the derived CATE estimator.

Beginning with counterfactual regression~\citep{cfr} and its remarkable performance, there are various attempts that handle the selection bias by minimizing the distribution discrepancy between the treatment groups in the representation space~\cite{site,disentangle,mim,ace}.
However, two critical issues with these methods have long been neglected, which impedes them from handling the treatment selection bias.
The first issue is the mini-batch sampling effects (MSE).
Specifically, current representation-based methods compute the distribution discrepancy within mini-batches instead of the entire data space, making it vulnerable to bad sampling cases.
For example, the presence of an outlier in a mini-batch can falsely increase the estimate of the distribution discrepancy, thereby misleading the update of the estimator.
The second issue is the unobserved confounder effects (UCE).
Specifically, current approaches mainly operate under the unconfoundedness assumption~\cite{kddbest} and ignore the pervasive influence of unobserved confounders, which makes the resulting estimators biased given the existence of unobserved confounders.

\paragraph{Contributions and outline.} 
In this paper, we propose an effective CATE estimator based on optimal transport, namely Entire Space CounterFactual Regression (ESCFR), which handles both the MSE and UCE issues with a generalized sinkhorn discrepancy.
Specifically, after presenting preliminaries in Section~\ref{sec:preliminary}, we redefine the problem of CATE estimation within the framework of stochastic optimal transport in Section~\ref{sec:sot}.
We subsequently showcase the MSE issue faced by existing approaches in Section~\ref{sec:uot} and introduce a mass-preserving regularizer to counter this issue. 
In Section~\ref{sec:jdot}, we explore the UCE issue and propose a proximal factual outcome regularizer to mitigate its impact. The architecture and learning objectives of ESCFR are elaborated upon in Section~\ref{sec:archi}, followed by the presentation of our experimental results in Section~\ref{sec:experiments}.

\section{Preliminaries}\label{sec:preliminary}

\subsection{Causal inference from observational data}\label{sec:causal}
This section formulates preliminaries in observational causal inference.
We first formalize the fundamental elements in Definition~\ref{def:rv} following the general notation convention\footnote{We use uppercase letters, \eg  $X$ to denote a random variable, and lowercase letters, \eg $x$ to denote an associated specific value.
Letters in calligraphic font, \eg $\mathcal{X}$ represent the support of the corresponding random variable,
and $\mathbb{P}()$ represents the probability distribution of the random variable, \eg $\mathbb{P}(X)$.}.
\begin{definition}\label{def:rv}
    Let $X$ be the random variable of covariates, with support $\mathcal{X}$ and distribution $\p(x)$;
    Let $R$ be the random variable of induced representations, with support $\mathcal{R}$ and distribution $\p(r)$;
    Let $Y$ be the random variable of outcomes, with support $\mathcal{Y}$ and distribution $\p(y)$;
    Let $T$ be the random variable of treatment indicator, with support $\mathcal{T}=\{0,1\}$ and distribution $\p(T)$.
\end{definition}
In the potential outcome framework~\citep{pof}, a unit characterized by the covariates $x$ has two potential outcomes, namely $Y_1(x)$ given it is treated and $Y_0(x)$ otherwise.
The CATE is defined as the conditionally expected difference of potential outcomes as follow:
\begin{equation}\label{eq:ite1}
    \tau(x) := \E\left[Y_1 - Y_0 \mid  x \right],
\end{equation}

\begin{figure*}
\centering
\subfigure[Handling treatment selection bias with $\psi(\cdot)$.]{\includegraphics[width=0.45\linewidth]{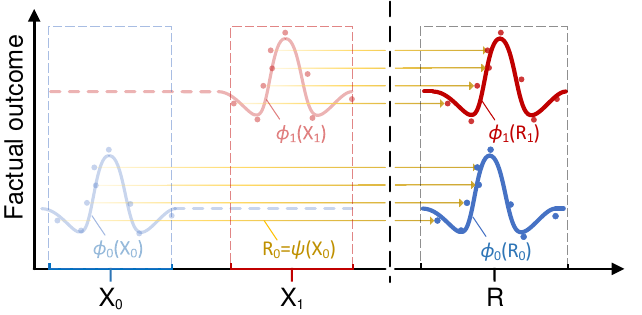}\label{fig:problem}}\quad
\subfigure[Architecture of ESCFR.]{\includegraphics[width=0.45\linewidth]{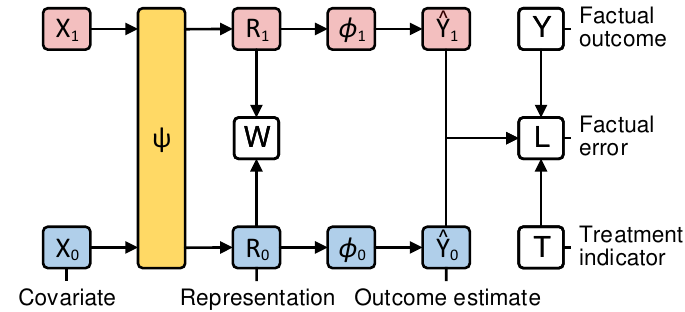}\label{fig:structure}}
\caption{Overview of handling treatment selection bias with ESCFR. The red and blue colors signify the treated and untreated groups, respectively.
(a) The treatment selection bias manifests as a distribution shift between $X_1$ and $X_0$. The scatters and curves represent the units and the fitted outcome mappings.
(b) ESCFR mitigates the selection bias by aligning units from different treatment groups in the representation space: $R=\psi(X)$, which enables $\phi_1$ and $\phi_0$ to generalize across groups.}
\end{figure*}
where one of these two outcomes is always unobserved in the collected data. 
To address such missing counterfactuals, the CATE estimation task is commonly decomposed into outcome estimation subproblems that are solvable with supervised learning method~\citep{kunzel2019metalearners}.
For example, T-learner models the factual outcomes $Y$ for units in the treated and untreated groups separately;
S-learner regards $T$ as a covariate, and then models $Y$ for units in all groups simultaneously.
The CATE estimate is then the difference of the estimated outcomes when $T$ is set to treated and untreated.

\begin{definition}\label{def:map}
    Let $\psi:\mathcal{X}\rightarrow\mathcal{R}$ be a mapping from support $\mathcal{X}$ to $\mathcal{R}$, \ie $\forall x \in \mathcal{X}$, $r=\psi(x)\in\mathcal{R}$.
    Let $\phi:\mathcal{R}\times\mathcal{T}\rightarrow\mathcal{Y}$ be a mapping from support $\mathcal{R}\times\mathcal{T}$ to $\mathcal{Y}$, \ie it maps the representations and treatment indicator to the corresponding factual outcome.
    Denote $Y_1=\phi_1(R)$ and $Y_0=\phi_0(R)$, where we abbreviate $\phi(R,T=1)$ and $\phi(R,T=0)$ to $\phi_1(R)$ and $\phi_0(R)$, respectively, for brevity.
\end{definition}

TARNet~\citep{cfr} combines T-learner and S-learner to achieve better performance, which consists of a representation mapping $\psi$ and an outcome mapping $\phi$ as defined in Definition~\ref{def:map}.
For an unit with \textcolor{black}{covariates $x$}, TARNet estimates the CATE as:
\begin{equation}
    \hat{\tau}_{\psi,\phi}(x)=\hat{Y}_1-\hat{Y}_0,\quad
    \hat{Y}_1=\phi_1(\psi(x)), \quad \hat{Y}_0=\phi_0(\psi(x)),
\end{equation}
where $\psi$ is trained over all individuals, $\phi_1$ and $\phi_0$ are trained over the treated and untreated group, respectively.
Finally, the quality of CATE estimation is evaluated with the precision in estimation of heterogeneous effect (PEHE) metric:
\begin{equation}\label{eq:pehe}
    \epehe(\psi,\phi):=\int_{\mathcal{X}}\left(\hat{\tau}_{\psi,\phi}(x)-\tau(x)\right)^{2} {\color{black}\mathbb{P}(x)}\,dx.
\end{equation}
However, according to Figure~\ref{fig:problem}, the treatment selection bias causes a distribution shift of covariates across groups, which misleads $\phi_1$ and $\phi_0$ to overfit their respective group's properties and generalize poorly to the entire population.
Therefore, the CATE estimate $\hat{\tau}$ by these methods would be biased.

\subsection{Discrete optimal transport and Sinkhorn divergence}\label{sec:ot}
Optimal transport (OT) quantifies distribution discrepancy as the minimum transport cost, offering a tool to quantify the selection bias in Figure~\ref{fig:problem}.
\citet{monge1781memoire} first formulated OT as finding an optimal mapping between two distributions.
\citet{kantorovich2006translocation} further proposed a more applicable formulation in Definition~\ref{def:kanto}, which can be seen as a generalization of the Monge problem.

\begin{definition}\label{def:kanto}
    For empirical distributions $\alpha$ and $\beta$ with n and m units, respectively, the Kantorovich problem aims to find a feasible plan $\pi\in\mathbb{R}_{+}^{n\times m}$ which transports $\alpha$ to $\beta$ at the minimum cost:
    \begin{equation}\label{eq:kanto}
        \mathbb{W}(\alpha,\beta):=\min_{\bpi\in\Pi(\alpha,\beta)}\left<\mathbf{D},\mathbf{\bpi}\right>,\;
        \Pi(\alpha,\beta):=\left\{\mathbf{\bpi}\in\mathbb{R}_{+}^{n\times m}: \bpi\mathbf{1}_m=\mathbf{a},\bpi^\mathrm{T}\mathbf{1}_n=\mathbf{b}\right\},
    \end{equation}
    where $\mathbb{W}(\alpha,\beta)\in\mathbb{R}$ is the Wasserstein discrepancy between $\alpha$ and $\beta$;
    $\mathbf{D}\in\mathbb{R}_{+}^{n\times m}$ is the unit-wise distance between $\alpha$ and $\beta$, which is implemented with the squared Euclidean distance;
    $\mathbf{a}$ and $\mathbf{b}$ indicate the mass of units in $\alpha$ and $\beta$,
    and $\Pi$ is the feasible transportation plan set which ensures the mass-preserving constraint holds.
\end{definition}

Since exact solutions to~\eqref{eq:kanto} 
often come with high computational costs~\citep{exact1}, researchers would 
commonly add an entropic regularization to the Kantorovich problem as follow:
\begin{equation}\label{eq:eot}
    \mathbb{W}^\epsilon(\alpha,\beta):=\left<\mathbf{D},\bpi^\epsilon\right>, \;
    \bpi^\epsilon:=\mathop{\arg\min}_{\bpi\in\Pi(\alpha,\beta)}\left<\mathbf{D},\mathbf{\bpi}\right>-\epsilon \mathrm{H}(\bpi), \;
    \mathrm{H}(\bpi):=-\sum_{i,j}\bpi_{i,j}\left(\log(\bpi_{i,j})-1\right),
\end{equation}
which makes the problem $\epsilon$-convex and solvable with the Sinkhorn algorithm~\citep{sinkhorn}.
The Sinkhorn algorithm only consists of matrix-vector products, making it suited to be accelerated with GPUs.

\section{Proposed method}\label{sec:proposed}

In this section, we present the Entire Space CounterFactual Regression (ESCFR) approach, which leverages optimal transport to tackle the treatment selection bias.
We first illustrate the stochastic optimal transport framework for distribution discrepancy quantification across treatment groups, and demonstrate its efficacy for improving the performance of CATE estimators. 
Based on the framework, we then propose a relaxed mass-preserving regularizer to address the sampling effect, and a proximal factual outcome regularizer to handle the unobserved confounders. We finally open a new thread to summarize the model architecture, learning objectives, and optimization algorithm.

\subsection{Stochastic optimal transport for counterfactual regression}\label{sec:sot}
Representation-based approaches mitigate the treatment selection bias by calculating distribution discrepancy in the representation space and then minimizing it.
OT is a preferred method to quantify the discrepancy due to its numerical advantages and flexibility over competitors.
It is numerically stable in the cases where other methods, such as $\phi$-divergence (\eg Kullback-Leibler divergence), fails~\citep{sot}.
Compared with adversarial discrepancy measures~\citep{deepMatch,ganite,adv}, the it can be calculated efficiently and integrated naturally with the traditional supervised learning framework.

Formally, we denote the OT discrepancy between treatment groups as $\mathbb{W}\left(\pt_\psi(r),\pc_\psi(r) \right)$, where $\pt_\psi(r)$ and $\pc_\psi(r)$ are the distributions of representations in treated and untreated groups, respectively, induced by the mapping $r=\psi\left(x\right)$.
The discrepancy is differentiable with respect to $\psi$~\citep{pot}, and thus can be minimized by updating the representation mapping $\psi$ with gradient-based optimizers.

\begin{definition}\label{def:empirical}
    Let {\color{black}$\hat{\mathbb{P}}^{T=1}(x):=\{x_i^{T=1}\}_{i=1}^n$ and $\hat{\mathbb{P}}^{T=0}(x):=\{x_i^{T=0}\}_{i=1}^m$} be the empirical distributions of covariates at a mini-batch level, which contain $n$ treated units and $m$ untreated units, respectively; $\hat{\mathbb{P}}_\psi^{T=1}(r)$ and $\hat{\mathbb{P}}_\psi^{T=0}(r)$ be that of representations induced by the mapping $r=\psi(x)$ in Definition~\ref{def:map}.
\end{definition}

However, since prevalent neural estimators mainly update parameters with stochastic gradient methods, only a fraction of the units is accessible within each iteration.
A shortcut in this context is to calculate the discrepancy at a stochastic mini-batch level:
\begin{equation}\label{eq:mbw}
 \hat{\mathbb{W}}_\psi:=\mathbb{W}\left(\hat{\mathbb{P}}^\mathrm{T=1}_\psi(r),\hat{\mathbb{P}}^\mathrm{T=0}_\psi(r) \right).
\end{equation}

The effectiveness of this shortcut is investigated by Theorem~\ref{thm:bound} (refer to Appendix~\ref{apdx_thm} for proof), which demonstrates that the PEHE can be optimized by iteratively minimizing the estimation error of factual outcomes and the mini-batch group discrepancy in~\eqref{eq:mbw}.

\begin{thm}\label{thm:bound}
    Let $\psi$ and $\phi$ be the representation mapping and factual outcome mapping, respectively;
    $\hat{\mathbb{W}}_\psi$ be the group discrepancy at a mini-batch level.
    With the probability of at least $1-\delta$, we have:
    \begin{equation}\label{eq:peheBound}
        \epehe(\psi,\phi)
        \leq 2[
            \epsilon^{T=1}_\mathrm{F}(\psi,\phi) + \epsilon^{T=0}_\mathrm{F}(\psi,\phi) + B_\psi \hat{\mathbb{W}}_\psi
            -2\sigma^2_{Y}+\mathcal{O}(\frac{1}{\delta N})
            ],
    \end{equation}
    where $\epsilon^{T=1}_\mathrm{F}$ and $\epsilon^{T=0}_\mathrm{F}$ are the expected errors of factual outcome estimation,
    $N$ is the batch size, $\sigma^2_Y$ is the variance of outcomes, 
    $B_\psi$ is a constant term, and $\mathcal{O}(\cdot)$ is a sampling complexity term.
\end{thm}

\subsection{Relaxed mass-preserving regularizer for sampling effect}\label{sec:uot}

\begin{wrapfigure}{r}{6.5cm}
\centering
\vspace{-5mm}
\subfigure[Ideal]{\includegraphics[width=0.28\linewidth,trim=13 10 10 10]{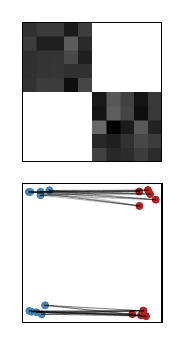}\label{fig:match:a}}\hspace{3.7mm}
\subfigure[Imbalance]{\includegraphics[width=0.28\linewidth,trim=13 10 10 10]{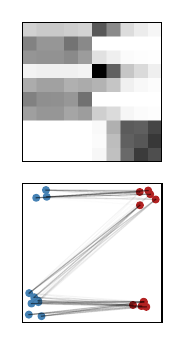}\label{fig:match:b}}\hspace{3.7mm}
\subfigure[Outlier]{\includegraphics[width=0.28\linewidth,trim=13 10 10 10]{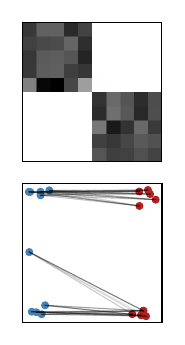}\label{fig:match:c}}
\caption{Optimal transport plan (upper) and its geometric interpretation (down) in three cases, where the connection strength represents the transported mass.
Different colors and vertical positions indicate different treatments and outcomes, respectively.}
\label{fig:match}
\vspace{-3mm}
\end{wrapfigure}

Although Theorem~\ref{thm:bound} guarantees that the empirical OT discrepancy~\eqref{eq:mbw} bounds the PEHE, the sampling complexity term $\mathcal{O}(\cdot)$ inspires us to investigate the potential risks of bad cases caused by stochastic sampling.
Precisely, the term $\mathcal{O}(\cdot)$ results from the discrepancy between the entire population and the sampled mini-batch units (see \eqref{eq:empirical1} and \eqref{eq:empirical3} in Appendix~\ref{apx_A}) which is highly dependent on the uncontrollable sampling quality. Therefore, a reliable discrepancy measure should be robust to bad sampling cases, otherwise the resulting vulnerability would impede us from quantifying and minimizing the actual discrepancy.

We consider three sampling cases in Figure~\ref{fig:match}, where the transport strategy is reasonable and applicable in the ideal sampling case in Figure~\ref{fig:match:a}.
However, the transport strategy calculated with~\eqref{eq:mbw} is vulnerable to disturbances caused by non-ideal sampling conditions.
For instance, Figure \ref{fig:match:b} reveals false matching of units with disparate outcomes in the sampled mini-batch where the outcomes of the two groups are imbalanced; Figure~\ref{fig:match:c} showcases false matching of the mini-batch outliers with normal units, causing a substantial disruption of the transport strategy for other units.
Therefore, the vanilla OT in~\eqref{eq:mbw} fails to quantify the group discrepancy for producing erroneous transport strategies in non-ideal mini-batches, thereby misleading the update of the representation mapping $\psi$. We term this issue the mini-batch sampling effect (MSE).

Furthermore, this MSE issue is not exclusive to OT-based methods but is also prevalent in other representation-based approaches. Despite this, OT offers a distinct advantage: it allows the formalization of the MSE issue through its mass-preservation constraint, as indicated in \eqref{eq:eot}. This constraint mandates a match for every unit in both groups, regardless of the real-world scenarios, complicating the transport of normal units and the computation of true group discrepancies. 
This issue is further exacerbated by small batch sizes. Such formalizability through mass-preservation provides a lever for handling the MSE issue, setting OT apart from other representation-based methods~\cite{ace,site}.

An intuitive approach to mitigate the MSE issue is to relax the marginal constraint, \ie to allow for the creation and destruction of the mass of each unit.
To this end, inspired by unbalanced and weak transport theories~\citep{uot,uotSink}, a relaxed mass-preserving regularizer (RMPR) is devised in Definition~\ref{def:uot}. The core technical point is to replace the hard marginal constraint in \eqref{eq:kanto} with a soft penalty in \eqref{eq:rmpr} for constraining transport strategy. 
On the basis, the stochastic discrepancy is calculated as
\begin{equation}
\hat{\mathbb{W}}_\psi^{\epsilon,\kappa}:=\mathbb{W}^{\epsilon,\kappa}\left(\hat{\mathbb{P}}^\mathrm{T=1}_\psi(r),\hat{\mathbb{P}}^\mathrm{T=0}_\psi(r) \right),
\end{equation}
where the hard mass-preservation constraint is removed to mitigate the MSE issue.
Building on Lemma 1 in \citet{mbuot}, we derive Corollary \ref{thm:robust} to investigate the robustness of RMPR to sampling effects, showcasing that the effect of mini-batch outliers is upper bounded by a constant.
\begin{definition}\label{def:uot}
    For empirical distributions $\alpha$ and $\beta$ with n and m units, respectively, optimal transport with relaxed mass-preserving constraint seeks the transport strategy $\bpi$ at the minimum cost:
    \begin{equation}\label{eq:rmpr}
        \mathbb{W}^{\epsilon,\kappa}(\alpha,\beta):=\left<\mathbf{D},\bpi\right>, 
        \mathbf{\bpi}:=\arg\min_{\bpi}\left<\mathbf{D},\mathbf{\bpi}\right> - \epsilon\mathrm{H}(\bpi)
        +\kappa(\mathrm{D_{KL}}(\bpi\mathbf{1}_m, \mathbf{a}) + \mathrm{D_{KL}}(\bpi^\mathrm{T}\mathbf{1}_n, \mathbf{b})  )
    \end{equation}
    where
    $\mathbf{D}\in\mathbb{R}_{+}^{n\times m}$ is the unit-wise distance, and
    $\mathbf{a}$, $\mathbf{b}$ indicate the mass of units in $\alpha$ and $\beta$, respectively.
\end{definition}

\begin{cor}\label{thm:robust}
    For empirical distributions $\alpha,\beta$ with $n$ and $m$ units, respectively, adding an outlier $a^\prime$ to $\alpha$ and denoting the disturbed distribution as $\alpha^\prime$,  we have
    \begin{equation}
        \mathbb{W}^{0,\kappa}\left(\alpha^\prime,\beta \right) - \mathbb{W}^{0,\kappa}\left(\alpha,\beta \right)
        \leq 2\kappa(1-e^{-\sum_{b\in\beta}(a^\prime-b)^2/2\kappa})/{\color{black}(n+1)},
    \end{equation}
    which is upper bounded by $2\kappa/{\color{black}(n+1)}$. $\mathbb{W}^{0,\kappa}$ is the unbalanced discrepancy as per Definition~\ref{def:uot}.
\end{cor}

In comparison with alternative methods~\citep{partialot,ruot} to relax the marginal constraint, the RMPR implementation in Definition~\ref{sec:uot} enjoys a collection of theoretical properties~\citep{uotSink} and can be calculated via the generalized Sinkhorn algorithm~\citep{uot}.
The calculated discrepancy is differentiable \wrt $\psi$ and thus can be minimized via stochastic gradient methods in an end-to-end manner.

\subsection{Proximal factual outcome regularizer for unobserved confounders}\label{sec:jdot}
\begin{wrapfigure}{r}{6.5cm}
\vspace{-5mm}
\centering
\subfigure[]{\includegraphics[width=0.4\linewidth]{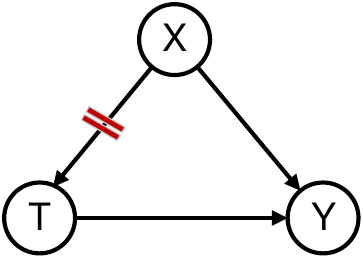}\label{fig:causalgraph:a}}\hspace{2mm}
\subfigure[]{\includegraphics[width=0.4\linewidth]{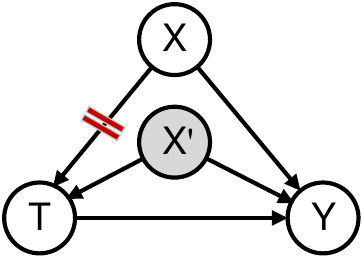}\label{fig:causalgraph:b}}
\vspace{-2mm}
\caption{Causal graphs with (a) and w/o (b) the unconfoundedness assumption. The shaded node indicates the hidden confounder $X^\prime$.}\label{fig:causalgraph}
\vspace{-3mm}
\end{wrapfigure}
Existing representation-based methods fail to eliminate the treatment selection bias due to the unobserved confounder effects (UCE).
Beginning with CFR~\citep{cfr}, the unconfoundedness assumption~\ref{asm:unconfoundedness} (see Appendix~\ref{apx_A}) is often taken to circumvent the UCE issue~\citep{wanginstrument}.
In this context, given two units $r_i\in\mathbb{P}_\psi^{T=1}(r)$ and $r_j\in\mathbb{P}_\psi^{T=0}(r)$, for instance, OT discrepancy in Definition~\ref{def:uot} calculates the unit-wise distance as $\mathbf{D}_{ij}:=\Vert r_i-r_j\Vert^2$.
If Assumption~\ref{asm:unconfoundedness} holds, it eliminates the treatment selection bias since it blocks the backdoor path $X\rightarrow T$ in Figure~\ref{fig:causalgraph:a} by balancing confounders in the latent space.
However, Assumption~\ref{asm:unconfoundedness} is usually violated due to the existence of unobserved confounders as per Figure~\ref{fig:causalgraph:b}, which hinders existing methods from handling treatment selection bias since the backdoor path $X^\prime\rightarrow T$ is not blocked. The existence of unobserved covariates $X^\prime$ also makes the transport strategy in vanilla OT unidentifiable, which invalidates the calculated discrepancy.

To mitigate the effect of $X^\prime$, we introduce a modification to the unit-wise distance. Specifically, we observe that given balanced $X$ and identical $T$, the only variable reflecting the variation of $X^\prime$ is the outcome $Y$.
As such, resonating with~\citet{jdot}, we calibrate the unit-wise distance $\mathbf{D}$ with potential outcomes as follow:
\begin{equation}\label{eq:pfor0}
    \mathbf{D}^{\gamma}_{ij} = \Vert r_i-r_j\Vert^2
    +\gamma\cdot\left[\Vert y_i^{T=0} - y_j^{T=0}\Vert^2+\Vert y_j^{T=1} - y_i^{T=1}\Vert^2\right],
\end{equation}
where $\gamma$ controls the strength of regularization.
The underlying regularization is straightforward: units with similar (observed and unobserved) confounders should also have similar potential outcomes. 
As such, for a pair of units with similar observed covariates, \ie $\Vert r_i - r_j\Vert^2\approx 0$, if their potential outcomes given the same treatment $t=\{0,1\}$ differ greatly, \ie $\Vert y_i^{t}-y_j^{t}\Vert>>0$, their unobserved confounders should likewise differ significantly.
The vanilla OT technique in~\eqref{eq:mbw} where $\mathbf{D}_{ij}=\Vert r_i-r_j\Vert^2$ would incorrectly match this pair, generate a false transport strategy, and consequently misguide the update of the representation mapping $\psi$.
In contrast, OT based on $\mathbf{D}^{\gamma}_{ij}$ would not match this pair as the difference of unobserved confounders is compensated with that of potential outcomes.

Moving forward, since $y_i^{T=0}$ and $y_j^{T=1}$ in~\eqref{eq:pfor0} are unavailable due to the missing counterfactual outcomes, the proposed proximal factual outcome regularizer (PFOR) uses their estimates instead. Specifically, let $\hat{y}_i$ and $\hat{y}_j$ be the estimates of $y_i^{T=0}$ and $y_j^{T=1}$, respectively, PFOR refines~\eqref{eq:pfor0} as
\begin{equation}\label{eq:pfor}
    \mathbf{D}^{\gamma}_{ij} = \Vert r_i-r_j\Vert^2+\gamma\cdot\left[\Vert\hat{y}_i - y_j\Vert^2+\Vert\hat{y}_j - y_i\Vert^2\right], \quad
    \hat{y}_i=\phi_0(r_i), \quad
    \hat{y}_j=\phi_1(r_j),
\end{equation}
Additional justifications, assumptions and limitations of PFOR are discussed in Appendix~\ref{sec:appendix_pfor}.

\subsection{Architecture of entire space counterfactual regression}\label{sec:archi}

The architecture of ESCFR is presented in Figure~\ref{fig:structure}, where the covariate $X$ is first mapped to the representations $R$ with $\psi(\cdot)$, and then to the potential outcomes with $\phi(\cdot)$.
The group discrepancy $\mathbb{W}$ is calculated with the optimal transport equipping with the RMPR in~\eqref{eq:rmpr} and PFOR in~\eqref{eq:pfor}.

The learning objective is to minimize the risk of factual outcome estimation and the group discrepancy.
Given mini-batch distributions $\hat{\mathbb{P}}^{T=1}(x)$ and $\hat{\mathbb{P}}^{T=0}(x)$ in Definition~\ref{def:empirical}, the risk of factual outcome estimation following~\cite{dragonnet} can be formulated as
\begin{equation}
        \mathcal{L}_\mathrm{F}(\psi,\phi):=\mathbb{E}_{x_i\in \hat{\mathbb{P}}^{T=1}(x)}\Vert\phi_1(\psi(x_i))-y_i\Vert^2
        + \mathbb{E}_{x_j\in \hat{\mathbb{P}}^{T=0}(x)}\Vert\phi_0(\psi(x_j))-y_j\Vert^2,
\end{equation}
where $y_i$ and $y_j$ are the factual outcomes for the corresponding treatment groups.
The discrepancy is:
\begin{equation}\label{eq:escfr_}
    \textcolor{black}{\mathcal{L}_{{\color{black}\mathrm{D}^\gamma}}^{\epsilon,\kappa}(\psi):=\mathbb{W}^{\epsilon,\kappa,\gamma}\left(\hat{\mathbb{P}}^\mathrm{T=1}_\psi(r),\hat{\mathbb{P}}^\mathrm{T=0}_\psi(r) \right),}
\end{equation}
which is in general the optimal transport with RMPR in Definition~\ref{def:uot}, except for the unit-wise distance $\mathbf{D}^{\gamma}$ calculated with the PFOR in~\eqref{eq:pfor}. 
Finally, the overall learning objective of ESCFR is
\begin{equation}\label{eq:escfr}
    \mathcal{L}_\mathrm{ESCFR}^{\epsilon,\kappa,\gamma,\lambda} := \mathcal{L}_\mathrm{F}(\psi,\phi) + \lambda\cdot\mathcal{L}_{{\color{black}\mathrm{D}^\gamma}}^{\epsilon,\kappa}(\psi),
\end{equation}
where $\lambda$ controls the strength of distribution alignment,
$\epsilon$ controls the entropic regularization in \eqref{eq:eot},
$\kappa$ controls RMPR in \eqref{eq:rmpr},
and $\gamma$ controls PFOR in \eqref{eq:pfor}.
The learning objective~\eqref{eq:escfr} mitigates the selection bias following Theorem~\ref{thm:bound} and handles the MSE and UCE issues.
 
The optimization procedure consists of three steps as summarized in Algorithm~\ref{alg:escfr} (see Appendix~\ref{apx_b}). 
First, compute $\bpi^{\epsilon,\kappa,\gamma}$ by solving the OT problem in Definition~\ref{def:uot} with Algorithm~\ref{alg:uotsinkhorn} (see Appendix~\ref{apx_b}), where the unit-wise distance is calculated with $\mathbf{D}^\gamma$.
Second, compute the discrepancy in \eqref{eq:escfr_} as $\left<\bpi^{\epsilon,\kappa,\gamma},\mathbf{D}^\gamma\right>$, which is differentiable to $\psi$ and $\phi$.
Finally, calculate the overall loss in \eqref{eq:escfr} and update $\psi$ and $\phi$ with stochastic gradient methods.

\section{Experiments}\label{sec:experiments}
\subsection{Experimental setup}
\paragraph{Datasets.}

Missing counterfactuals impede the evaluation of PEHE using observational benchmarks. Therefore, experiments are conducted on two semi-synthetic benchmarks~\cite{cfr,site}, \ie IHDP and ACIC. Specifically, IHDP is designed to estimate the effect of specialist home visits on infants' potential cognitive scores, with 747 observations and 25 covariates. ACIC comes from the collaborative perinatal project~\citep{acic}, and includes 4802 observations and 58 covariates.

\paragraph{Baselines.}
We consider three groups of baselines.
(1) Statistical estimators:
    least square regression with treatment as covariates (OLS),
    random forest with treatment as covariates (R.Forest),
    a single network with treatment as covariates (S.learner~\cite{kunzel2019metalearners}),
    separate neural regressors for each treatment group (T.learner~\cite{kunzel2019metalearners}),
    TARNet~\citep{cfr};
(2) Matching estimators:
    propensity score match with logistic regression (PSM~\cite{psm}),
    k-nearest neighbor (k-NN~\cite{knn}),
    causal forest (C.Forest~\cite{dmlforest}),
    orthogonal forest (O.Forest~\cite{dmlforest});
(3) Representation-based estimators:
    balancing neural network (BNN~\cite{bnn}),
    counterfactual regression with MMD (CFR-MMD) and Wasserstein discrepancy (CFR-WASS)~\cite{cfr}.
\paragraph{Training protocol.} A fully connected neural network with two 60-dimensional hidden layers is selected to instantiate the representation mapping $\psi$ and the factual outcome mapping $\phi$ for ESCFR and other neural network based baselines. 
To ensure a fair comparison, all neural models are trained for a maximum of 400 epochs using the Adam optimizer, with the patience of early stopping being 30. The learning rate and weight decay are set to 1$e^{-3}$ and $1e^{-4}$, respectively.
Other settings of optimizers follow~\citet{adam}.
We fine-tune hyperparameters within the range in Figure~\ref{fig:param}, validate performance every two epochs, and save the optimal model for test.

\paragraph{Evaluation protocol.} The PEHE in \eqref{eq:pehe} is the primary metric for performance evaluation~\cite{site,cfr}.
However, it is unavailable in the model selection phase due to missing counterfactuals.
As such, we use the area under the uplift curve (AUUC)~\citep{uplift} to guide model selection, which evaluates the ranking performance of the CATE estimator and can be computed without counterfactual outcomes. Although AUUC is not commonly used in treatment effect estimation, we report it as an auxiliary metric for reference.
The within-sample and out-of-sample results are computed on the training and test set, respectively, following the common settings~\cite{site,otite3,causalot,cfr}.

\subsection{Overall performance}
\begin{table*}[]
\caption{Performance (mean±std) on the PEHE and AUUC metrics. ``*'' marks the baseline estimators that ESCFR outperforms significantly at p-value $<$ 0.05 over paired samples t-test.}\label{tab:result}
\resizebox{\linewidth}{!}{
\begin{tabular}{lccccccccccc}
\toprule
Dataset    & \multicolumn{2}{c}{ACIC (PEHE)}    && \multicolumn{2}{c}{IHDP (PEHE)}       && \multicolumn{2}{c}{ACIC (AUUC)}     && \multicolumn{2}{c}{IHDP (AUUC)}\\ \cmidrule{2-3} \cmidrule{5-6} \cmidrule{8-9} \cmidrule{11-12}
Model     & In-sample     & Out-sample   && In-sample    & Out-sample      && In-sample     & Out-sample   && In-sample    & Out-sample \\ \midrule
OLS       & 3.749±0.080\sig   & 4.340±0.117\sig  && 3.856±6.018  & 5.674±9.026     && 0.843±0.007   & 0.496±0.017\sig  && 0.652±0.050  & 0.492±0.032\sig\\
R.Forest  & 3.597±0.064\sig   & 3.399±0.165\sig  && 2.635±3.598  & 4.671±9.291     && 0.902±0.016   & 0.702±0.026\sig  && 0.736±0.142  & 0.661±0.259\\
S.Learner & 3.572±0.269\sig   & 3.636±0.254\sig  && 1.706±1.600\sig  & 3.038±5.319     && \textbf{0.905±0.041}   & 0.627±0.014\sig  && 0.633±0.183  & 0.702±0.330\\
T.Learner & 3.429±0.142\sig   & 3.566±0.248\sig  && 1.567±1.136\sig  & 2.730±3.627     && 0.846±0.019   & 0.632±0.020\sig  && 0.651±0.179  & 0.707±0.333\\
TARNet    & 3.236±0.266\sig   & 3.254±0.150\sig  && 0.749±0.291  & 1.788±2.812     && 0.886±0.046   & 0.662±0.014\sig  && 0.654±0.184  & 0.711±0.329\\
\midrule
C.Forest  & 3.449±0.101\sig   & 3.196±0.177\sig  && 4.018±5.602\sig  & 4.486±8.677     && 0.717±0.005\sig   & 0.709±0.018\sig  && 0.643±0.141  & 0.695±0.294\\
k-NN      & 5.605±0.168\sig   & 5.892±0.138\sig  && 2.208±2.233\sig  & 4.319±7.336     && 0.892±0.007\sig   & 0.507±0.034\sig  && 0.725±0.142  & 0.668±0.299\\
O.Forest  & 8.094±4.669\sig   & 4.148±2.224\sig  && 2.605±2.418\sig  & 3.136±5.642     && 0.744±0.013   & 0.699±0.022\sig  && 0.664±0.157  & 0.702±0.325\\
PSM       & 5.228±0.154\sig   & 5.094±0.301\sig  && 3.219±4.352\sig  & 4.634±8.574     && 0.884±0.010   & 0.745±0.021  && \textbf{0.740±0.149}  & 0.681±0.253\\
\midrule
BNN       & 3.345±0.233\sig   & 3.368±0.176\sig  && 0.709±0.330  & 1.806±2.837     && 0.882±0.033   & 0.645±0.013\sig  && 0.654±0.184  & 0.711±0.329\\
CFR-MMD   & 3.182±0.174\sig   & 3.357±0.321\sig  && 0.777±0.327  & 1.791±2.741     && 0.871±0.032   & 0.659±0.017\sig  && 0.655±0.183  & 0.710±0.329\\
CFR-WASS  & 3.128±0.263\sig   & 3.207±0.169\sig  && 0.657±0.673  & 1.704±3.115     && 0.873±0.029   & 0.669±0.018\sig  && 0.656±0.187  & 0.715±0.311\\
\midrule
ESCFR     & \textbf{2.252±0.297}   & \textbf{2.316±0.613}  && \textbf{0.502±0.252}  & \textbf{1.282±2.312} 
         && 0.796±0.030   & \textbf{0.754±0.021}  && 0.665±0.166  & \textbf{0.719±0.311}\\
\bottomrule

\end{tabular}
}
\end{table*}

Table~\ref{tab:result} compares ESCFR and its competitors. Main observations are noted as follows.
\begin{itemize}[leftmargin=*]
    \item Statistical estimators demonstrate competitive performance on the PEHE metric, with neural estimators outperforming linear and random forest methods due to their superior ability to capture non-linearity. In particular, TARNet, which combines the advantages of T-learner and S-learner, achieves the best overall performance among statistical estimators. However, the circumvention to treatment selection bias results in inferior performance.
    \item Matching methods such as PSM exhibit strong ranking performance, which explains their popularity in counterfactual ranking practice. However, their relatively poor performance on the PEHE metric limits their applicability in counterfactual estimation applications where accuracy of CATE estimation is prioritized, such as advertising systems.
    \item Representation-based methods mitigate the treatment selection bias and enhance overall performance. 
    In particular, CFR-WASS reaches an out-of-sample PEHE of 3.207 on ACIC, advancing most statistical methods. 
    However, it utilizes the vanilla Wasserstein discrepancy, wherein the MSE and UCE issues impede it from solving the treatment selection bias.
    The proposed ESCFR achieves significant improvement over most metrics compared with various prevalent baselines\footnote{An exception would be the within-sample AUUC, which is reported over training data and thus easy to be overfitted. 
    This metric is not critical as the factual outcomes are typically unavailable in the inference phase.
    We mainly rely on out-of-sample AUUC instead to evaluate the ranking performance and perform model selection.
    }. 
    Combined with the comparisons above, we attribute its superiority to the proposed RMPR and PFOR regularizers, which accommodate ESCFR to the situations where MSE and UCE exist. 
\end{itemize}

\subsection{Ablation study}\label{sec:ablation}
\begin{table*}[]
\caption{Ablation study (mean±std) on the ACIC benchmark. ``*'' marks the variants that ESCFR outperforms significantly at p-value $<$ 0.01 over paired samples t-test.}\label{tab:ablation}
\centering
\resizebox{\linewidth}{!}{
\setlength{\tabcolsep}{18pt}
\begin{tabular}{cccccccc}
\toprule
    &&& \multicolumn{2}{c}{In-sample} &  & \multicolumn{2}{c}{Out-sample} \\ \cmidrule{1-3} \cmidrule{4-5} \cmidrule{7-8}
SOT&RMPR&PFOR      & PEHE    & AUUC   && PEHE     & AUUC  \\ \midrule
\textcolor{lightgray}{\XSolidBrush}&\textcolor{lightgray}{\XSolidBrush}&\textcolor{lightgray}{\XSolidBrush}    & 3.2367±0.2666\sig   & \textbf{0.8862+0.0462}  && 3.2542+0.1505\sig  & 0.6624+0.0149\sig \\
\Checkmark&\textcolor{lightgray}{\XSolidBrush}&\textcolor{lightgray}{\XSolidBrush}      & 3.1284±0.2638\sig   & 0.8734+0.0291  && 3.2073+0.1699\sig  & 0.6698+0.0187\sig \\
\Checkmark&\Checkmark&\textcolor{lightgray}{\XSolidBrush}        & 2.6459+0.2747\sig   & 0.8356+0.0286  && 2.7688±0.4009  & 0.7099+0.0157\sig \\
\Checkmark&\textcolor{lightgray}{\XSolidBrush}&\Checkmark        & 2.5705±0.3403\sig   & 0.8270±0.0341  && 2.6330±0.4672   & 0.7110±0.0287\sig \\
\Checkmark&\Checkmark&\Checkmark          & \textbf{2.2520±0.2975}   & 0.7968±0.0307  && \textbf{2.3165+0.6136}  & \textbf{0.7542±0.0202} \\

\bottomrule

\end{tabular}
}
\end{table*}

In this section, to further verify the effectiveness of individual components, an ablation study is conducted on the ACIC benchmark in Table~\ref{tab:ablation}.
Specifically, ESCFR first augments TARNet with the stochastic optimal transport to align the confounders in the representation space, as described in Section~\ref{sec:sot}, which reduces the out-of-sample PEHE from 3.254 to 3.207.
Subsequently, it mitigates the MSE issue with RMPR as per Section~\ref{sec:uot} and the UCE issue with PFOR as per Section~\ref{sec:jdot}, which reduces the out-of-sample PEHE to 2.768 and 2.633, respectively.
Finally, ESCFR combines the RMPR and PFOR in a unified framework as detailed in Section~\ref{sec:archi}, which further reduces the value of PEHE and advances the best performance of other variants.

\subsection{Analysis of relaxed mass-preserving regularizer}
\begin{wrapfigure}{r}{6.5cm}
\centering
\vspace{-8mm}
\subfigure[$\kappa=50$]{\includegraphics[width=0.28\linewidth,trim=13 10 8 10]{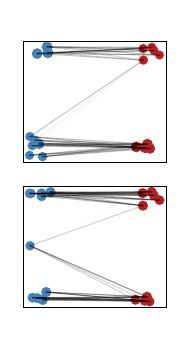}}\hspace{3.7mm}
\subfigure[$\kappa=10$]{\includegraphics[width=0.28\linewidth,trim=13 10 8 10]{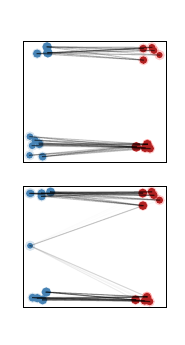}}\hspace{3.7mm}
\subfigure[$\kappa=2$]{\includegraphics[width=0.28\linewidth, trim=13 10 8 10]{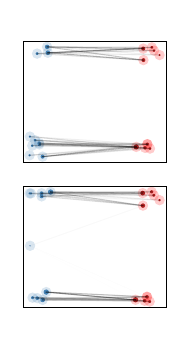}}
\caption{Geometric interpretation of OT plan with RMPR under the outcome imbalance (upper) and outlier (down) settings. The dark area indicates the transported mass of a unit, \ie marginal of the transport matrix $\pi$. The light area indicates the total mass. }
\label{fig:rmpr}
\end{wrapfigure}
Most prevalent methods fail to handle the label imbalance and mini-batch outliers in Figure~\ref{fig:match}(b-c).
Figure~\ref{fig:rmpr} 
shows the transport plan generated with RMPR in the same situations, where RMPR alleviates the MSE issue in both bad cases.
Initially, RMPR with $\kappa=50$ presents similar matching scheme with Figure~\ref{fig:match}, since in this setting the loss of marginal constraint is strong, and the solution is thus similar to that of the vanilla OT problem in Definition~\ref{def:kanto};
RMPR with $\kappa=10$ further looses the marginal constraint and avoids the incorrect matching of units with different outcomes;
RMPR with $\kappa=2$ further gets robust to the outlier's interference and correctly matches the remaining units.
This success is attributed to the relaxation of the mass-preserving constraint according to Section~\ref{sec:uot}.

Notably, RMPR does not transport all mass of a unit. The closer the unit is to the overlapping zone in a batch, the greater the mass is transferred. That is, RMPR adaptively matches and pulls closer units that are close to the overlapping region, ignoring outliers, which mitigates the bias of causal inference methods in cases where the positivity assumption does not strictly hold.
Current approaches mainly achieve it by manually cleaning the data or dynamically weighting the units~\citep{genebound}, while RMPR naturally implements it via the soft penalty in \eqref{eq:rmpr}.

We further investigate the performance of RMPR under different batch sizes and $\kappa$ in Appendix~\ref{sec:appendix_rmpr} to verify the effectiveness of RMPR more thoroughly.

\subsection{Parameter sensitivity study}\label{sec:parameter}
\begin{figure*}
    \centering
    \includegraphics[width=\linewidth, trim=10 30 10 0]{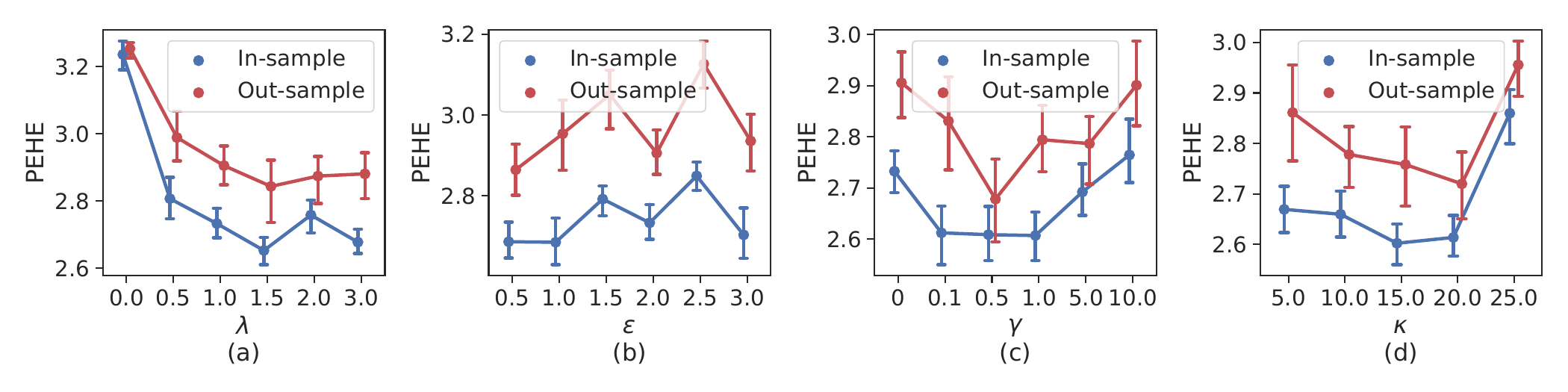}
    \caption{Parameter sensitivity of ESCFR, where the lines and error bars indicate the mean values and 90\% confidence intervals, respectively. (a) Impact of alignment strength ($\lambda$). (b) Impact of entropic regularization strength $\epsilon$. (c) Impact of PFOR strength $\gamma\ (\times10^{3})$. (d) Impact of RMPR strength $\kappa$.}
    \label{fig:param}
\end{figure*}

In this section, we investigate the role of four critical hyperparameters within the ESCFR framework—namely, $\lambda$, $\epsilon$, $\kappa$, and $\gamma$. These hyperparameters have a profound impact on learning objectives and the performance, as substantiated by the experimental results presented in Figure~\ref{fig:param}.

First, we vary $\lambda$ to investigate the efficacy of stochastic optimal transport. Our findings suggest that a gradual increase in the value of $\lambda$ leads to consistent improvements. For instance, the out-of-sample PEHE diminishes from 3.22 at $\lambda=0$ to roughly 2.85 at $\lambda=1.5$. Nonetheless, overly emphasizing distribution balancing within a multi-task learning framework can result in compromised factual outcome estimation and, consequently, suboptimal CATE estimates.
The next parameter scrutinized is $\epsilon$, which serves as an entropic regularizer. Our study shows that a larger $\epsilon$ accelerates the computation of optimal transport discrepancy~\citep{pot}. However, this computational advantage comes at the cost of a skewed transport plan, as evidenced by an increase in out-of-sample PEHE with oscillations.

Following this, we explore the hyperparameters $\gamma$ and $\kappa$, associated with PFOR and RMPR, respectively. Extending upon vanilla optimal transport, we incorporate unit-wise distance via \eqref{eq:pfor} and assess the impact of PFOR. Generally, the inclusion of PFOR positively influences CATE estimation. However, allocating excessive weight to proximal outcome distance compromises the performance, as it dilutes the role of the unit-wise distance in the representation space. 
Concurrently, we modify the transport problem with \eqref{eq:rmpr} to examine the implications of RMPR. Relieving the mass-preserving constraints through RMPR significantly enhances CATE estimation, but an excessively low value of $\kappa$ hampers performance as it fails to ensure that the representations across treatment groups are pulled closer together in the optimal transport paradigm.

\section{Related works}\label{sec:relate}

Current research aims to alleviate treatment selection bias through the balancing of distributions between treated and untreated groups. These balancing techniques can be broadly categorized into three categories: reweighting-based, matching-based, and representation-based methods.

Reweighting-based methods mainly employ propensity scores to achieve global balance between groups. The core procedure comprises two steps: the estimation of propensity scores and the construction of unbiased estimators.
Propensity scores are commonly estimated using logistic regression~\cite{mtlips,dual,dai2022generalized,chen2021autodebias}. To enhance the precision of these estimates, techniques such as feature selection~\cite{shortreed2017outcome,wang2023out,wang2023select}, joint optimization~\cite{mtlips,zhanguser}, and alternative training techniques~\cite{zhu2020unbiased} have been adopted.
The unbiased estimator is exemplified by the inverse propensity score method~\cite{ips}, which inversely re-weights each units with the estimated propensity scores. While theoretically unbiased, it suffers from high variance with low propensity and bias with incorrect propensity estimates~\cite{escm,lhxpropensity,li2022multiple}. 
To alleviate these drawbacks, doubly robust estimators and variance reduction techniques have been introduced~\cite{dr,mrdr,li2023stabledr}. 
Nevertheless, these methods remain constrained by their reliance on propensity scores, affecting their efficacy in real-world applications.

Matching-based methods aim to match comparable units from different groups to construct locally balanced distributions.
The key distinctions between representative techniques~\citep{psm,nnm,rnnm,wustable} lie in their similarity measures.
A notable exemplar is the propensity score matching approach~\cite{psm}, which computes unit (dis)similarity based on estimated propensity scores. 
Notably, Tree-based methods~\citep{dmlforest} can also be categorized as matching approaches, but use adaptive similarity measures.
However, these techniques are computationally intensive, limiting their deployment in large-scale operations.

Representation-based methods aim to construct a mapping to a feature space where distributional discrepancies are minimized. The central challenge lies in the accurate calculation of such discrepancies. Initial investigations focused on maximum mean discrepancy and vanilla Wasserstein discrepancy \cite{bnn,cfr}, later supplemented by local similarity preservation~\cite{site,ace}, feature selection~\cite{disentangle,mim}, representation decomposition~\cite{wuite,disentangle} and adversarial training \cite{ganite} mechanisms.
Despite their success, they fail under specific but prevalent conditions, such as outlier fluctuations~\cite{mbuot} and unlabeled confounders~\cite{zheng2021sensitivity}, undermining the reliability of calculated discrepancies.

The recent exploration of optimal transport in causality~\cite{ot4discovery} has spawned innovative reweighting-based~\cite{otite1}, matching-based~\cite{otite4} and representation-based methods~\cite{otite3,causalot}. For example, \citet{causalot} utilize OT to align factual and counterfactual distributions; \citet{otite3} and \citet{kddbest} use OT to reduce confounding bias. Despite these advancements, they largely adhere to the vanilla Kantorovich problem that corresponds to the canonical Wasserstein discrepancy, akin to \citet{cfr}. Adapting OT problems to meet the unique needs of treatment effect estimation remains an open area for further research.

%% file: appendix.tex
\section{Causal Inference with Observational Studies} \label{apx_A}
In this section, we introduce necessary preliminaries about causal inference and treatment effect estimation, for readers that are unfamiliar with this area. We then present our theoretical insights based on these preliminaries.
\subsection{Problem Formulation}\label{apdx_assump}
This section formalizes the definitions, assumptions, and useful lemmas in causal inference from observational data.
Following the notations in Section~\ref{sec:causal}, an individual with covariates $x$ has two potential outcomes, namely $Y_1(x)$ given it is treated and $Y_0(x)$ otherwise.
The ground-truth individual treatment effect (CATE) is the difference in its potential outcomes.
\begin{definition}\label{def:teA}
The individual treatment effect (CATE) for a unit with covariates $x$ is
\begin{equation}
    \label{eq:ite}
    \tau(x):= \E\left[Y_1 - Y_0 \mid  x \right],
\end{equation}
where we abbreviate $Y_1(x)$ to $Y_1$ for brevity.
The expectation is over the potential outcome space $\mathcal{Y}$.
\end{definition}
Estimating CATE with observational data is a common practice in causal inference, which has long been confronted with two primary challenges:
\begin{itemize}[leftmargin=*]
    \item Missing counterfactuals: where only the factual outcome is observable. If a patient is treated, for instance, we can never observe what would have happened if the patient was untreated in the same situation.
    \item Treatment selection bias, where individuals have preferences for treatment selection.
    For example, doctors would adapt different treatment plans for patients with different health conditions.
    It would make the treated and untreated populations heterogeneous.
    CATE estimators na\"ively trained to minimize the factual outcome error would overfit the respective group's properties and thus cannot generalize well to the entire population.
\end{itemize}
\citet{pearl2018book} suggested a two-step methodology to overcome these two challenges.
The first step is identification, which aims to construct an unbiased statistical estimand to identify the causal estimand (\eg $\tau(x)$) based on the adjustment formula.
Note that not all causal estimands are identifiable, \eg CATE is identifiable only if Assumption~\ref{asm:unconfoundedness}-\ref{asm:sutva} hold.
\begin{asmp}\label{asm:unconfoundedness} (Unconfoundedness). For all covariates $x$ in the population of interest (\ie $x$ with $\p(X=x)>0$), we have conditional independence $(Y_0,Y_1)\indep T\mid X=x$. That is, potential outcomes are conditionally independent of treatment assignment.
\end{asmp}
\begin{asmp}\label{asm:consistency} (Consistency). For all covariates $x$ in the population of interest, we have $Y=Y_t$.
That is, the observed outcome is consistent with the potential outcome \wrt the assigned treatment.
\end{asmp}
\begin{asmp}\label{asm:positivity} (Positivity). For all covariates $x$ in the population of interest, we have $0<\p(T=1\mid X=x)<1$.
That is, all individuals have a chance to be assigned both treatments.
\end{asmp}
\begin{asmp}\label{asm:sutva} (SUTVA). The potential outcomes for any unit are not affected by the treatment assignments of other units, and there are no different forms or versions of each treatment level for each unit that can produce different potential outcomes~\cite{imbens2015causal}.
\end{asmp}
The second step is estimation, which aims to estimate the derived statistical estimand with observational data.
Lemma~\ref{lem:ite} illustrates how this two-step approach can be used for CATE estimation.

\begin{lem}\label{lem:ite}
The CATE estimand $\tau(x)$ can be identified as:
\begin{equation}
    \begin{aligned}
        \E\left[Y_1-Y_0\mid X=x\right]
        &=\E\left[Y_1 \mid X=x\right]-\E\left[Y_0\mid X=x\right] \\
        &\overset{(1)}{=}\E\left[Y_1\mid X=x,T=1\right]-\E\left[Y_0\mid X=x, T=0\right] \\
        &\overset{(2)}{=}\E\left[Y\mid X=x,T=1\right]-\E\left[Y\mid X=x, T=0\right],
    \end{aligned}
\end{equation}
where (1) stems from the unconfoundedness assumption~\ref{asm:unconfoundedness};
(2) stems from the consistency assumption~\ref{asm:consistency}.
The derived estimand is fully composed of statistical estimands, which can only be estimated under the positivity assumption~\ref{asm:positivity}.
Otherwise, if the positivity assumption is violated, we have:
\begin{equation}
    \begin{aligned}
        \E\left[Y\mid X=x,T=1\right]
        &=\int y\cdot\p(Y=y \mid X=x,T=1)\, dy\\
        &=\int y\cdot\frac{\p(Y=y, X=x, T=1)}{\p(T=1 \mid X=x) \p(X=x)}\, dy,
    \end{aligned}
\end{equation}
which is not computable as there exists $x\in\mathcal{X}$ which makes $\p(T=1 \mid X=x)=0$.
\end{lem}

\subsection{Meta-learners for CATE estimation with observational data}
In an effort to solve missing counterfactuals, existing meta-learner based methods~\citep{kunzel2019metalearners,rlearner} decompose the CATE estimation problem into several subproblems that can be solved with any supervised learning method.
As depicted in Figure~\ref{fig:metas}, S-learner regards the treatment indicator $T$ as one of the covariates $X$, and utilizes the shared representation mapping $\psi$ and outcome mapping $\phi$ to estimate the factual outcomes.
However, because the network structure does not highlight the role of treatment indicator, it may be overlooked when treatment effects are minimal.
T-learner models the factual outcomes for treated units $X_1$ and untreated units $X_0$ separately, which highlights the treatment indicator's effect; however, it reduces the data efficiency and is therefore inapplicable when the dataset is small.
\citet{kunzel2019metalearners} discuss the advantages and limitations of these two approaches in more detail.

\begin{definition}\label{def:mapapp}
    Let $\psi:\mathcal{X}\rightarrow\mathcal{R}$ be a mapping from support $\mathcal{X}$ to $\mathcal{R}$.
    That is, $\forall x \in \mathcal{X}$, $\exists r=\psi(x)\in\mathcal{R}$.
    Let $\phi:\mathcal{R}\times\mathcal{T}\rightarrow\mathcal{Y}$ be a mapping from support $\mathcal{R}\times\mathcal{T}$ to $\mathcal{Y}$.
    That is, it maps the representations and treatment indicator to the corresponding factual outcome.
    For example, $Y_1=\phi_1(R)$, $Y_0=\phi_0(R)$, where we will always abbreviate $\phi(R,T=1)$ and $\phi(R,T=0)$ to $\phi_1(R)$ and $\phi_0(R)$, respectively.
\end{definition}
\begin{asmp}
    $\phi:\mathcal{X}\rightarrow\mathcal{R}$ is differentiable and invertible, with its inverse $\phi^{-1}$ defined over $\mathcal{R}$.
\end{asmp}
\begin{figure}
\centering
\subfigure[S-learner 
]{\includegraphics[width=0.3\linewidth,trim=0 0 0 0]{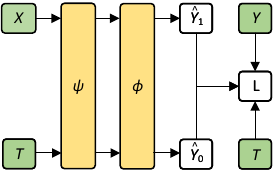}}\hspace{6mm}
\subfigure[T-learner 
]{\includegraphics[width=0.3\linewidth,trim=0 0 0 0]{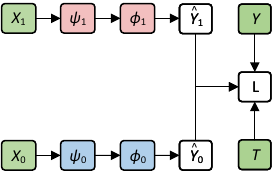}}\hspace{6mm}
\subfigure[TARNet 
]{\includegraphics[width=0.3\linewidth,trim=0 0 0 0]{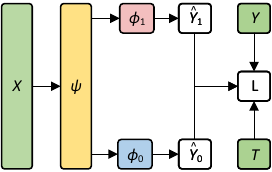}}
\caption{Architecture of Meta-learner based CATE estimators, consisting of inputs (green), outputs (white), shared mappings (yellow), and mappings for treated and untreated units (red and blue, respectively).}\label{fig:metas}
\end{figure}

TARNet~\citep{cfr} in Figure~\ref{fig:metas} (c) obtains better results by absorbing the advantages of both T-learner and S-learner, which consists of a representation mapping $\psi$ and an outcome mapping $\phi$ as defined in Definition~\ref{def:mapapp}.
For a unit with covariates $X$, TARNet estimates CATE as the difference in predicted outcomes when $T$ is set to treated and untreated:
\begin{equation}
    \hat{\tau}_{\psi,\phi}(X):=\hat{Y}_1-\hat{Y}_0,\quad \mathrm{where}\quad \hat{Y}_1=\phi_1(\psi(X)), \quad \hat{Y}_0=\phi_0(\psi(X)),
\end{equation}
where $\psi$ is trained over all units, $\phi_1$ and $\phi_0$ are trained over the treated and untreated units, respectively, to minimize the factual error $\epsilon_\mathrm{F}(\phi,\psi)$ in Definition~\ref{def:perunitlossA}.
Finally, the performance of the CATE estimator is mainly evaluated with PEHE:
\begin{equation}\label{eq:pehe_app}
    \epehe(\psi,\phi)=\int_{\mathcal{X}}\left(\hat{\tau}_{\psi,\phi}(x)-\tau(x)\right)^{2}\mathbb{P}(x)\,dx.
\end{equation}

\begin{definition}\label{def:perunitlossA}
Let $\mathbb{L}$ be the loss function that measures the quality of outcome estimation, \eg the squared loss.
The expected loss for the units with covariates $x$ and treatment indicator $t$ is:
\begin{equation}
    \loss(x,t):= \int_\cY \mathbb{L}(Y_t,\phi(\psi(x),t))\cdot\mathbb{P}(Y_t\mid x)\,dY_t.
\end{equation}
    where $\mathbb{L}$ is realized with the squared loss: $\mathbb{L}(Y_t,\psi(\phi(x),t))=(Y_t-\psi(\phi(x),t))^2$ in our scenario.
The expected factual outcome estimation error for treated, untreated and all units are:
\begin{equation}
    \label{eq:error1}
    \begin{aligned}
        \epsilon^\mathrm{T=1}_\mathrm{F}(\psi,\phi)  &:= \int_\mathcal{X}{\loss(x,1)\cdot\pt(x)\, dx}, \\
        \epsilon^\mathrm{T=0}_\mathrm{F}(\psi,\phi)  &:= \int_{\mathcal{X}}  \loss(x,0)\cdot\pc(x)\, dx,\\
        \epsilon_\mathrm{F}(\psi,\phi)  &:= \int_{\mathcal{X} \times \mathcal{T}}  \loss(x,t)\cdot\mathbb{P}(x,t)\, dxdt .
    \end{aligned}
\end{equation}
\end{definition}
\subsection{Representation-based Methods for Treatment Selection Bias}\label{A3}

However, the treatment selection bias makes covariate distributions across groups shift.
As such, $\phi_1$ and $\phi_0$ would overfit the respective group's properties and thus cannot generalize well to the entire population.
For example, as shown in Figure~\ref{fig:problem}, the potential outcome estimator $\phi_1$ trained with treated units cannot generalize to the untreated units.
Therefore, the resulting $\hat{\tau}$ would be biased.

\begin{definition}\label{def:distri}
Let $\pt(x) := \p(x\mid T=1)$ and $\pc(x) := \p(x\mid T=0)$ be the covariate distribution for treated and untreated groups, respectively.
Let $\pt_\psi(r)$ and $\pc_\psi(x)$ be that of representations induced by the representation mapping $r=\psi(x)$ defined in Definition~\ref{def:map}.
\end{definition}

To mitigate the effect of treatment selection bias, representation-based approaches~\citep{bnn,cfr} minimize the distribution discrepancy of different groups in the representation space.
In particular, the integral probability metric (IPM) in Definition~\ref{def:distri} is a widely used metric that measures the discrepancy of two distributions.
\citet{cfr} propose to optimize the PEHE by minimizing the estimation error of factual outcomes $\epsilon_\mathrm{F}$ and the IPM of learned representations between treated and untreated groups.
They further provide theoretical results to back up their claim as per Theorem~\ref{thm:indtausqloss}.


%
\begin{definition}\label{def:ipm}
Consider two distribution functions $\pt(x)$ and $\pc(x)$ supported over $\mathcal{X}$,
let $\mathcal{F}$ be a sufficiently large function family,
the integral probability metric induced by $\mathcal{F}$ is
\begin{equation}
    \mathrm{IPM}_\mathcal{F}\left(\pt,\pc\right) = \sup_{f\in \mathcal{F}} \left|\int_\mathcal{X} f(x) \left(\pt(x) - \pc(x)\right) \, dx \right|,
\end{equation}
\end{definition}

\begin{thm}\label{thm:indtausqloss}
Let $\psi$ and $\phi$ be the mappings in Definition~\ref{def:map},
$\mathcal{F}$ be a predefined sufficiently large function family of $\phi$, $\mathrm{IPM}_\mathcal{F}$ be the integral probability metric induced by $\mathcal{F}$.
Assume there exists a constant $B_\psi>0$, such that for $t \in \{0,1\}$, $\frac{1}{B_\psi} \cdot \loss(x,t) \in \mathcal{F}$ holds.
\cite{cfr} demonstrate:
\begin{equation}
    \epehe(\psi,\phi) \leq 2\left(\epsilon_\mathrm{F}^\mathrm{T=0}(\psi,\phi) +\epsilon_\mathrm{F}^\mathrm{T=1}(\psi,\phi)+  B_\psi  \mathrm{IPM}_\mathcal{F} \left( \pt_\psi, \pc_\psi \right) - 2\sigma^2_Y\right),
\end{equation}
where $\epsilon_\mathrm{F}^\mathrm{T=0}$ and $\epsilon_\mathrm{F}^\mathrm{T=1}$ follow Definition~\ref{def:perunitlossA}, $\pt_\psi(r)$ and $\pc_\psi(x)$ follow Definition~\ref{def:distri}.
\end{thm}

\subsection{Theoretical Results and Extensions}\label{apdx_thm}
Two problems with Theorem~\ref{thm:indtausqloss} warrant further consideration.
Firstly, the IPM metric, albeit with profound theoretical properties, is intractable.
To counter this, note that the IPM holds for any sufficiently large function families, it is feasible to consider IPM in certain function families $\mathcal{F}$ to make it tractable. For example. in the $1$-Lipschitz function family, the IPM is equivalent to the Wasserstein divergence as per Kantorovich-Rubinstein duality~\cite{villani2009optimal,cfr}.
As such, the IPM discrepancy can be casted to the Wasserstein discrepancy for computation as per Lemma~\ref{lem:ipm}.
\begin{lem}\label{lem:ipm}
    Consider two distribution functions $\p_1(x)$ and $\p_2(x)$ supported over $\mathcal{X}$; let $\mathcal{F}$ be the family of $1$-Lipschitz functions, $\mathbb{W}$ be the Wasserstein distance, \citet{villani2009optimal} demonstrate
    \begin{equation}
        \mathrm{IPM}_\mathcal{F}\left(\p_1,\p_2\right) = \mathbb{W}\left(\p_1,\p_2\right)
    \end{equation}
\end{lem}

Another issue that needs further consideration is sampling complexity.
Specifically, Theorem~\ref{thm:indtausqloss} holds if and only if the entire populations of treated and untreated groups are available.
However, since the representation-based approaches update parameters with stochastic gradient methods, only a mini-batch of the population is accessible within each iteration.
As such, it remains questionable how does Theorem~\ref{thm:indtausqloss} perform at a mini-batch level in practice.

\begin{lem}
\label{lem:wd-emprical}
Let $\mathbb{P}(x)$ be a probability measure supported over $\mathcal{X}\in\mathbb{R}^d$
satisfying $T_1(\lambda)$ inequality.
Let $\hat{\mathbb{P}}(x) = \frac{1}{N} \sum_{i=1}^{N} \delta_{x_i}$ be the corresponding empirical measure with $N$ units.
\citet{hoeffding2} and \citet{hoeffding} demonstrate that for any $d^\prime>d$ and $\lambda^\prime< \lambda$, there exists some constant $N_0$, such that for any $\varepsilon>0$
\footnote{While there is a risk of symbol reuse, we use $\varepsilon$ here to denote sampling error, and $\epsilon$ to control the strength of entropic regularization in optimal transport.}
and $N \geq N_0\max(\varepsilon^{-(d+2)},1)$, we have
\begin{equation}
\mathbb{P}\left(\mathbb{W}\left(\mathbb{P}(x),\hat{\mathbb{P}}(x)\right)>\varepsilon\right) \leq \exp \left( - \frac{\lambda^\prime}{2} N \varepsilon^2  \right)
\end{equation}
where $d', \lambda'$ can be calculated explicitly.
\end{lem}

Hoeffding's inequality is a powerful statistical tool to quantify such sampling effects, which is proved to be applicable for $\mathbb{W}$ by \cite{hoeffding2}.
Therefore, it is natural to expand $\mathbb{W}$ according to Lemma~\ref{lem:wd-emprical} to extend Theorem~\ref{thm:indtausqloss} to mini-batch situations, in order to quantify the sampling effects.
\begin{thm}\label{thm:boundapp}
    Let $\psi$ and $\phi$ be the representation mapping and factual outcome mapping, respectively;
    $\hat{\mathbb{W}}_\psi$ be the discrepancy across groups at a mini-batch level.
    With the probability of at least $1-\delta$, we have:
    \begin{equation}\label{eq:peheBoundapp}
    \begin{aligned}
        \epehe(\psi,\phi)
        &\leq 2\left[
            \epsilon^\mathrm{T=1}_\mathrm{F}(\psi,\phi) + \epsilon^\mathrm{T=0}_\mathrm{F}(\psi,\phi) + B_\psi \hat{\mathbb{W}}_\psi -2\sigma^2_{Y}+\mathcal{O}(\frac{1}{\delta N})
            \right],
    \end{aligned}
    \end{equation}
    where $\epsilon^\mathrm{T=1}_\mathrm{F}$ and $\epsilon^\mathrm{T=0}_\mathrm{F}$ are the expected losses of factual outcome estimation over treated and untreated units, respectively.
    $N$ is the batch size, $\sigma^2_Y$ is the variance of outcomes, 
    $B_\psi$ is some constant such that $\frac{1}{B_\psi} \cdot \loss(x,t)$ belongs to the family of 1-Lipschitz functions, $\mathcal{O}(\cdot)$ is the sampling complexity term.
\end{thm}

\begin{proof}
    According to Theorem~\ref{thm:indtausqloss} we have:
    \begin{equation}
        \epehe(\psi,\phi) \leq 2\left(\epsilon_\mathrm{F}^\mathrm{T=0}(\psi,\phi) +\epsilon_\mathrm{F}^\mathrm{T=1}(\psi,\phi)+  B_\psi  \mathrm{IPM}_\mathcal{F} \left( \pt_\psi, \pc_\psi \right) - 2\sigma^2_Y\right) .
    \end{equation}

    Assuming that there exists a constant $B_\psi>0$, such that for $t \in \{0,1\}$, $\frac{1}{B_\psi} \cdot \loss(x,t)$ belongs to the family of 1-Lipschitz functions.
    According to Lemma~\ref{lem:ipm}, we have
    \begin{equation}\label{eq:peheBoundapp2}
        \epehe(\psi,\phi) \leq 2\left(\epsilon_\mathrm{F}^\mathrm{T=0}(\psi,\phi) +\epsilon_\mathrm{F}^\mathrm{T=1}(\psi,\phi)+  B_\psi  \mathbb{W}\left( \pt_\psi, \pc_\psi \right) - 2\sigma^2_Y\right) .
    \end{equation}

    Following Definition~\ref{def:empirical}, let $\hat{\mathbb{P}}_\psi^\mathrm{T=1}(r)$ and $\hat{\mathbb{P}}_\psi^\mathrm{T=0}(r)$ be the empirical distributions of representations at a mini-batch level, containing $N_1$ treated units and $N_0$ untreated units, respectively.
    Then we have:
    \begin{equation}\label{eq:empirical1}
        \begin{aligned}
            \mathbb{W}\left( \pt_\psi, \pc_\psi \right)
            &\leq \mathbb{W}\left( \pt_\psi, \hat{\mathbb{P}}_\psi^\mathrm{T=1} \right) + \mathbb{W}\left( \pc_\psi, \hat{\mathbb{P}}_\psi^\mathrm{T=1} \right)\\
            &\leq \mathbb{W}\left( \pt_\psi, \hat{\mathbb{P}}_\psi^\mathrm{T=1} \right) + \mathbb{W}\left( \pc_\psi, \hat{\mathbb{P}}_\psi^\mathrm{T=0} \right) + \mathbb{W}\left( \hat{\mathbb{P}}_\psi^\mathrm{T=0}, \hat{\mathbb{P}}_\psi^\mathrm{T=1} \right)\\
            :&= \mathbb{W}\left( \pt_\psi, \hat{\mathbb{P}}_\psi^\mathrm{T=1} \right) + \mathbb{W}\left( \pc_\psi, \hat{\mathbb{P}}_\psi^\mathrm{T=0} \right) + \hat{\mathbb{W}}_\psi,
        \end{aligned}
    \end{equation}
    because we have the triangular inequality for $\mathbb{W}$.
    The Hoeffding inequality in Lemma~\ref{lem:wd-emprical} further gives the following inequality which holds with the probability at least $1-\delta$:
    \begin{equation}\label{eq:empirical2}
        \begin{aligned}
           \mathbb{W}\left( \pt_\psi, \hat{\mathbb{P}}_\psi^\mathrm{T=1} \right)
           &\leq \sqrt{2 \log \left(\frac{1}{\delta} \right) / \lambda^\prime N_1} \\
           \mathbb{W}\left( \pc_\psi, \hat{\mathbb{P}}_\psi^\mathrm{T=0} \right)
           &\leq \sqrt{2 \log \left(\frac{1}{\delta} \right) / \lambda^\prime N_0}.
        \end{aligned}
    \end{equation}

    Denote $N:=N_0+N_1$ as the batch size, $\theta:=N_1/N$ as the ratio of treated units in the current batch.
    Combining \eqref{eq:empirical1} and~\ref{eq:empirical2} we have
    \begin{equation}\label{eq:empirical3}
        \begin{aligned}
            \mathbb{W}\left( \pt_\psi, \pc_\psi \right)
            &\leq \hat{\mathbb{W}}_\psi + \sqrt{2 \log \left(\frac{1}{\delta} \right) / \lambda^\prime N_1} + \sqrt{2 \log \left(\frac{1}{\delta} \right) / \lambda^\prime N_0}\\
            &=    \hat{\mathbb{W}}_\psi + \sqrt{2 \log \left(\frac{1}{\delta} \right) / \lambda^\prime N}  \left(\sqrt{\frac{1}{\theta}}+\sqrt{\frac{1}{1-\theta}}\right)\\
            :&=   \hat{\mathbb{W}}_\psi + \mathcal{O}(\frac{1}{\delta N}),
        \end{aligned}
    \end{equation}
    that holds with the probability at least $(1-\delta)^2$. $\mathcal{O}(\cdot)$ satisfies
    \begin{equation}\label{eq:obound}
        \sqrt{ \log \left(\frac{1}{\delta} \right) / \lambda^\prime }\left(1+\sqrt {1/(N-1)}\right)
    \geq\mathcal{O}(\frac{1}{\delta N})
    \geq 4\sqrt{ \log \left(\frac{1}{\delta} \right) / \lambda^\prime N},
    \end{equation}
    where $\mathcal{O}(\frac{1}{\delta N})$ reaches its maximum when $\theta=1/N$ or $\theta=1-1/N$, reaches its minimum when $\theta=0.5$.
    This corollary can be derived by differentiating the function $f(x)=1/\sqrt{x}+1/\sqrt{1-x}$.

    Combining \eqref{eq:peheBoundapp2} and~\eqref{eq:empirical3}, with the probability at least $(1-\delta)^2$, we have
   \begin{equation}
    \begin{aligned}
        \epehe(\psi,\phi)
        &\leq 2\left[
            \epsilon^\mathrm{T=1}_\mathrm{F}(\psi,\phi) + \epsilon^\mathrm{T=0}_\mathrm{F}(\psi,\phi) + B_\psi \hat{\mathbb{W}}_\psi -2\sigma^2_{Y}+\mathcal{O}(\frac{1}{\delta N})
            \right],
    \end{aligned}
    \end{equation}
    where we denote $B_\psi\mathcal{O}(\frac{1}{\delta N})$ as $\mathcal{O}(\frac{1}{\delta N})$.
    Finally, it is straightforward to derive the probabilistic approximately correct format that holds with probability at least $(1-\delta^\prime)$ by setting $\delta=1-\sqrt{1-\delta^{\prime}}$,
    and the proof is completed.
\end{proof}
Theorem~\ref{thm:boundapp} extends Theorem~\ref{thm:indtausqloss} and derives the upper bound of PEHE in the stochastic batch form, which demonstrates that the PEHE can be optimized by iteratively minimizing the factual outcome estimation error and the optimal transport discrepancy \emph{at a mini-batch level}.
\begin{cor}
    The empirical variance of the PEHE estimates in \eqref{eq:peheBoundapp} largely depends on the batch size and the proportion of treated and untreated units.
    Large batch size and balanced proportion correspond to low empirical variance, and vice versa.
\end{cor}
\begin{proof}
    It can be drawn directly from \eqref{eq:peheBoundapp} (batch size) and \eqref{eq:obound} (treatment proportion).
\end{proof}

\begin{cor}\label{thm:robustapp}
    For discrete measures
$\alpha = \sum_{i=1}^n \mathbf{a}_i \delta_{\mathbf{x}_i}$ and
$\beta = \sum_{j=1}^m \mathbf{b}_j \delta_{\mathbf{x}_j}$, adding an outlier $\delta_{\mathbf{x}^\prime}$ to $\alpha$ and denote the disturbed distribution as $\alpha^\prime$, we have
    \begin{equation}\label{coro_uot}
        \mathbb{W}^{0,\kappa}\left(\alpha^\prime,\beta \right) - \mathbb{W}^{0,\kappa}\left(\alpha,\beta \right)
        \leq 2\kappa(1-e^{-\sum_{j=1}^m(\mathbf{x}^\prime-\mathbf{x}_j)^2/2\kappa})/{\color{black}(n+1)},
    \end{equation}
    which is upper bounded by $2\kappa/(n+1)$. $\mathbb{W}^{0,\kappa}$ is the unbalanced discrepancy as per Definition~\ref{def:uot}.
\end{cor}
\begin{proof}
This is a direct extension to the Lemma 1 by~\citet{mbuot}, under the assumption that all the units including the outlier $\delta_{\mathbf{x}^\prime}$ share the same mass (\ie uniform mass distribution in each group). Specifically, when adding an outlier to $\alpha$ and obtaining a disturbed measure $\alpha^\prime$, the mass of each unit in $\alpha^\prime$ is $1/(n+1)$ (the OT problem would normalize the mass of units, \ie the total mass of the measure equals to 1). Based on this assumption, we set the $\zeta$ in the Lemma 1 by~\citet{mbuot} to $n/(n+1)$ and derived the Equation~(\ref{coro_uot}) with the denominator being $(n+1)$.
\end{proof}

\section{Discrete Optimal Transport}\label{apx_b}
This section proposes the definitions and algorithms to calculate optimal transport between discrete measures.
We have omitted the case of general measures~\cite{monge1781memoire} since it is beyond the scope of this work.
Instead, we provide an equivalent interpretation under discrete measures.
Readers interested in this topic should refer to~\cite{cot,sinkhorn} for details.
\subsection{Problem Formulation}

Consider $n$ warehouses and $m$ factories, where the $i$-th warehouse contains $\mathbf{a}_i$ units of materials; the $j$-th factory needs $\mathbf{b}_j$ units of materials~\cite{cot}.
Now we construct a \emph{mapping} from warehouses to factories, satisfying: (1) all materials of warehouses are transported; (2) all requirements of factories are satisfied; (3) materials from one warehouse are transported to \emph{no more than one} factory (mapping constraint).
Every feasible mapping is associated with a \emph{global} cost, calculated by aggregating the \emph{local} cost of moving a unit of material from the $i$-th warehouse to the $j$-th factory.
Our objective, to find a feasible mapping that minimizes the transport cost, is formulated in Definition~\ref{def:monge}.

\begin{definition}\label{def:monge}
For discrete measures
$\alpha = \sum_{i=1}^n \mathbf{a}_i \delta_{\mathbf{x}_i}$ and
$\beta = \sum_{j=1}^m \mathbf{b}_j \delta_{\mathbf{x}_j}$,
the Monge problem seeks for a mapping $\mathbb{T}:\{\mathbf{x}_i\}_{i=1}^{n}\rightarrow \{\mathbf{x}_j\}_{j=1}^{m}$ that associates to each point $\mathbf{x}_i$ a single point $\mathbf{x}_j$ and pushes the mass of $\alpha$ to $\beta$.
That is, $\forall j \in \{1,\dots, m\}$ we have $\mathbf{b}_j = \sum_{ i : \mathbb{T}(\mathbf{x}_i) = \mathbf{x}_j } \mathbf{a}_i$.
This mass-preserving constraint is abbreviated as $\mathbb{T}_\sharp \alpha = \beta$.
The mapping should also minimize the transportation cost denoted as $c(x,y)$.
To this end, Monge problem for discrete measures is formulated as:
\begin{equation}{\label{eq-monge-discr}
	\min_{\mathbb{T}: \mathbb{T}_\sharp \alpha = \beta}} \left\{ \sum_{i} c(\mathbf{x}_i,\mathbb{T}(\mathbf{x}_i)) \right\}.
\end{equation}
\end{definition}

This problem was further utilized to compare two probability measures where $\sum_i\mathbf{a}_i=\sum_j\mathbf{b}_j=1$.
However, Monge's formulation cannot guarantee the existence and uniqueness of solutions~\citep{cot}.
\cite{kantorovich2006translocation} relaxed the mapping constraint by allowing the transport from one warehouse to many factories and reformulated the Monge problem as a linear programming problem in Definition~\ref{def:kantoro}.

\begin{definition}\label{def:kantoro}
    For discrete measures
    $\alpha = \sum_{i=1}^n \mathbf{a}_i \delta_{\mathbf{x}_i}$ and
    $\beta = \sum_{j=1}^m \mathbf{b}_j \delta_{\mathbf{x}_j}$,
    the Kantorovich problem aims to find a feasible plan $\pi\in\mathbb{R}_{+}^{n\times m}$ which transports $\alpha$ to $\beta$ at minimum cost:
    \begin{equation}
        \mathbb{W}(\alpha,\beta):=\min_{\bpi\in\Pi(\alpha,\beta)}\left<\mathbf{D},\mathbf{\bpi}\right>,\quad
        \Pi(\alpha,\beta):=\left\{\mathbf{\bpi}\in\mathbb{R}_{+}^{n\times m}: \bpi\mathbf{1}_m=\mathbf{a},\bpi^\mathrm{T}\mathbf{1}_n=\mathbf{b}\right\},
    \end{equation}
    where $\mathbb{W}(\alpha,\beta)\in\mathbb{R}$ is the Wasserstein discrepancy between $\alpha$ and $\beta$;
    $\mathbf{D}\in\mathbb{R}_{+}^{n\times m}$ is the unit-wise distance\footnote{In this work, we calculate the unit-wise distance with the squared Euclidean metric following~\cite{otda}.} between $\alpha$ and $\beta$;
    $\mathbf{a}$ and $\mathbf{b}$ indicate the mass of units in $\alpha$ and $\beta$,
    and $\Pi$ is the feasible transportation plan set which ensures the mass-preserving constraint holds.
\end{definition}

\subsection{Sinkhorn Discrepancy and Algorithm}
\begin{algorithm}[tb]
\caption{Sinkhorn Algorithm}\label{alg:sinkhorn}
\textbf{Input}: discrete measures
    $\alpha = \sum_{i=1}^n \mathbf{a}_i \delta_{\mathbf{x}_i}$ and
    $\beta = \sum_{j=1}^m \mathbf{b}_j \delta_{\mathbf{x}_j}$,
distance matrix $\mathbf{D}_{ij} = \Vert\mathbf{x}_i-\mathbf{x}_j\Vert_2^2$.\\
\textbf{Parameter}: $\epsilon$: strength of entropic regularization; $\ell_\mathrm{max}$: maximum iterations.\\
\textbf{Output}: $\bpi^\epsilon$: the entropic regularized optimal transport matrix.
\begin{algorithmic}[1] 
\State $\mathbf{K} \gets \exp (-\mathbf{D}/\epsilon)$ 
\State $\mathbf{u} \gets \mathbf{1}_n$,
$\mathbf{v} \gets \mathbf{1}_m$, $\ell\gets 1$ 
\While{$\ell<\ell_\mathrm{max}$}
    \State $\mathbf{u} \gets \mathbf{a}/(\mathbf{K} \mathbf{v})$ 
    \State $\mathbf{v} \gets \mathbf{b}/(\mathbf{K}^{\mathrm{T}} \mathbf{u})$ 
    \State $\ell \gets \ell + 1$ 
    \EndWhile
    \State $\bpi^{\epsilon} \gets \operatorname{diag}(\mathbf{u}) \mathbf{K} \operatorname{diag}(\mathbf{v})$ 
\end{algorithmic}
\end{algorithm}

Exact solutions to the Kantorovich problem suffer from great computational costs.
The interior-point and network-simplex methods, for example, have a complexity of $\mathcal{O}(n^3\log n)$~\citep{exact3}.
A shortcut is to add an entropic regularizer as
\begin{equation}\label{eq:eotapp}
    \mathbb{W}^\epsilon(\alpha,\beta):=\left<\mathbf{D},\bpi^\epsilon\right>,\quad
    \bpi^\epsilon:=\mathop{\arg\min}_{\bpi\in\Pi(\alpha,\beta)}\left<\mathbf{D},\mathbf{\bpi}\right>-\epsilon \mathrm{H}(\bpi),\quad
    \mathrm{H}(\bpi):=-\sum_{i,j}\bpi_{ij}\left(\log(\bpi_{ij})-1\right),
\end{equation}
which makes the problem $\epsilon$-convex and solvable with the Sinkhorn algorithm~\citep{sinkhorn}, with a lower complexity of $\mathcal{O}(n^2/\epsilon^2)$.
Besides, the Sinkhorn algorithm consists of matrix-vector products only, which makes it suited to be accelerated with GPUs.
Specifically, let $\mathbf{f}\in\mathbb{R}^n$ and $\mathbf{g}\in\mathbb{R}^m$ be the lagrangian multipliers, the Lagrangian of \eqref{eq:eotapp} is:
\begin{equation}
\Phi(\bpi, \mathbf{f}, \mathbf{g})=\langle\mathbf{D}, \bpi\rangle-\epsilon \mathrm{H}(\bpi)-\left\langle\mathbf{f}, \bpi \mathbf{1}_{n}-\mathbf{a}\right\rangle-\left\langle\mathbf{g}, \bpi^{\mathrm{T}} \mathbf{1}_{m}-\mathbf{b}\right\rangle
\end{equation}

According to the first-order condition of constraint optimization problem, we have:
\begin{equation}
    \frac{\partial \Phi(\bpi, \mathbf{f}, \mathbf{g})}{\partial \bpi_{ij}}=\mathbf{D}_{ij}+\varepsilon \log \left(\bpi_{ij}\right)-\mathbf{f}_{i}-\mathbf{g}_{j}=0,
\end{equation}
or equivalently, the best transport matrix $\bpi^\epsilon$ should satisfy:
\begin{equation}
    \begin{aligned}
        \bpi_{ij}^{\epsilon}
        &=\exp\left( \frac{\mathbf{f}_{i}}{\epsilon} \right) *
    \exp\left( -\frac{\mathbf{D}_{ij}}{ \epsilon} \right) *
    \exp{\left( \frac{\mathbf{g}_{j}}{\epsilon} \right)}.
    \end{aligned}
\end{equation}

Let $\mathbf{u}_i:=\exp(\mathbf{f}_i/\epsilon)$, $\mathbf{v}_j:=\exp(\mathbf{g}_j/\epsilon)$, $\mathbf{K}_{ij}:=\exp(-\mathbf{D}_{ij}/\epsilon)$, then we have $\bpi^\epsilon=\mathrm{diag}(\mathbf{u})\mathbf{K}\mathrm{diag}(\mathbf{v})$.
The transport matrix should also satisfy the mass-preserving constraint, such that:
\begin{equation}
    \operatorname{diag}(\mathbf{u}) \mathbf{K} \operatorname{diag}(\mathbf{v}) \mathbf{1}_{m}=\mathbf{a}, \quad \quad \operatorname{diag}(\mathbf{v}) \mathbf{K}^{\top} \operatorname{diag}(\mathbf{u}) \mathbf{1}_{n}=\mathbf{b},
\end{equation}
or equivalently, let $\odot$ be the entry-wise multiplication of vectors, we have:
\begin{equation}\label{eq:pointwise}
    \mathbf{u} \odot(\mathbf{K} \mathbf{v})=\mathbf{a} \quad \text { and } \quad \mathbf{v} \odot\left(\mathbf{K}^{\mathrm{T}} \mathbf{u}\right)=\mathbf{b}.
\end{equation}

\eqref{eq:pointwise} is known as the matrix scaling problem.
An intuitive approach is to solve them iteratively:
\begin{equation}
    \mathbf{u}^{(\ell+1)} {=} \frac{\mathbf{a}}{\mathbf{K} \mathbf{v}^{(\ell)}} \quad \text { and } \quad \mathbf{v}^{(\ell+1)} {=} \frac{\mathbf{b}}{\mathbf{K}^{\mathrm{T}} \mathbf{u}^{(\ell+1)}}
\end{equation}
which is the critical step of Sinkhorn algorithm in Algorithm~\ref{alg:sinkhorn}.
The optimal transport matrix $\bpi^\epsilon$ acting as a constant matrix further induces the \emph{Sinkhorn discrepancy} $\mathbb{W}^\epsilon$ following \eqref{eq:eotapp}.
As $\mathbf{D}$ is differentiable to $\alpha$ and $\beta$, it is feasible to minimize $\mathbb{W}^\epsilon$ by adjusting the generation process of $\alpha$ and $\beta$, \ie the representation mapping in Definition~\ref{def:mapapp} with gradient-based optimizers.

\subsection{Unbalanced optimal transport and generalized sinkhorn}\label{sec:uotapp}
\begin{algorithm}[tb]
\caption{Generalized Sinkhorn Algorithm for Unbalanced Optimal Transport}\label{alg:uotsinkhorn}
\textbf{Input}: discrete measures
    $\alpha = \sum_{i=1}^n \mathbf{a}_i \delta_{\mathbf{x}_i}$ and
    $\beta = \sum_{j=1}^m \mathbf{b}_j \delta_{\mathbf{x}_j}$,
distance matrix $\mathbf{D}_{ij} = \Vert\mathbf{x}_i-\mathbf{x}_j\Vert_2^2$.\\
\textbf{Parameter}: $\epsilon$: strength of entropic regularizer; $\kappa$: strength of mass preserving; $\ell_\mathrm{max}$: max iterations.\\
\textbf{Output}: $\bpi^{\epsilon,\kappa}$: the entropic regularized unbalanced optimal transport matrix.
\begin{algorithmic}[1] 
\State $\mathbf{K} \gets \exp (-\mathbf{D}/\epsilon)$.
\State $\mathbf{f} \gets \mathbf{0}_n$, $\mathbf{g} \gets \mathbf{0}_m$, $\ell\gets 1$.
\While{$\ell<\ell_\mathrm{max}$}
    \State $\mathbf{u} \gets \exp(\mathbf{f}_i/\epsilon)$, $\mathbf{v} \gets \exp(\mathbf{g}_j/\epsilon)$
    \State $\bpi \gets \operatorname{diag}(\mathbf{u}) \mathbf{K} \operatorname{diag}(\mathbf{v})$.
    \State $\mathbf{a}^\prime \gets \bpi\mathbf{1}_n$, $\mathbf{b}^\prime \gets \bpi^\mathrm{T}\mathbf{1}_m$.
    \If{$\ell // 2 = 0$}
        \State  $\mathbf{f}\gets \left[\frac{\mathbf{f}}{\epsilon}+\log (\mathbf{a})-\log \left(\mathbf{a}^\prime\right)\right] \frac{\epsilon \kappa}{\epsilon+\kappa}$
    \Else
        \State $\mathbf{g}\gets \left[\frac{\mathbf{g}}{\epsilon}+\log (\mathbf{b})-\log \left(\mathbf{b}^\prime\right)\right] \frac{\epsilon \kappa}{\epsilon+\kappa}$
    \EndIf
    \State $\ell \gets \ell + 1$.
    \EndWhile
    \State $\bpi^{\epsilon, \kappa} \gets \operatorname{diag}(\mathbf{u}) \mathbf{K} \operatorname{diag}(\mathbf{v})$.
\end{algorithmic}
\end{algorithm}

We have reported the mini-batch sampling effect (MSE) issue of $\mathbb{W}^\epsilon$ in Section~\ref{sec:uot}, and attributed it to the mass-preserving constraint in \eqref{eq:eotapp}.
An intuitive approach to mitigate MSE is to relax the marginal constraint and allow for the creation and destruction of mass.
To this end, RMPR is proposed in Definition~\ref{def:uotapp}, which replaces the hard marginal constraint with a soft penalty.

\begin{definition}\label{def:uotapp}
    For empirical distributions $\alpha$ and $\beta$ with n and m units, respectively, unbalanced optimal transport seeks a transport plan at minimum cost:
    \begin{equation}\label{eq:rmprapp}
        \mathbb{W}^{\epsilon,\kappa}(\alpha,\beta):=\min_{\bpi}\left<\mathbf{D},\bpi\right>,
        \mathbf{\bpi}:=\arg\min_{\bpi}\left<\mathbf{D},\mathbf{\bpi}\right> + \epsilon\mathrm{H}(\bpi)+\kappa(\mathbf{K L}(\bpi\mathbf{1}_n, \mathbf{a}) + \mathbf{K L}(\bpi^\mathrm{T}\mathbf{1}_m, \mathbf{b})  ),
    \end{equation}
    where
    $\mathbf{D}\in\mathbb{R}_{+}^{n\times m}$ is the unit-wise distance, and
    $\mathbf{a}$ and $\mathbf{b}$ indicate the mass of units in $\alpha$ and $\beta$.
\end{definition}
The unbalanced optimal transport problem in Definition~\ref{def:uotapp} has a similar structure with~\eqref{eq:eotapp} and thus can be solved with a generalized Sinkhorn algorithm \citep{uot}.
The derivation starts from the Fenchel-Legendre dual form of \eqref{eq:rmprapp}:
\begin{equation}
\begin{aligned}
\max _{\mathbf{f}\in\mathbb{R}^n, \mathbf{g} \in \mathbb{R}^{m}}&-F^{*}(-\mathbf{f})-G^{*}(-\mathbf{g})-\epsilon \sum_{i, j} \exp \left(\frac{\mathbf{f}_i+\mathbf{g}_j-\mathbf{D}_{i j}}{\epsilon}\right),\\
F^{*}(\mathbf{f})&=\max _{\mathbf{z} \in \mathbb{R}^{n}} \mathbf{z}^{\top} \mathbf{f}-\kappa \mathbf{K L}(\mathbf{z} \| \mathbf{a})=\kappa\left\langle e^{\mathbf{f} / \kappa}, \mathbf{a}\right\rangle-\mathbf{a}^{\top} \mathbf{1}_{n}, \\
G^{*}(\mathbf{g})&=\max _{\mathbf{z} \in \mathbb{R}^{m}} \mathbf{z}^{\top} \mathbf{g}-\kappa \mathbf{K L}(\mathbf{z} \| \mathbf{b})=\kappa\left\langle e^{\mathbf{g} / \kappa}, \mathbf{b}\right\rangle-\mathbf{b}^{\top} \mathbf{1}_{m},
\end{aligned}
\end{equation}
where the functions $F^*(\cdot)$ and $G^*(\cdot)$ are the Legendre transformation of KL divergence. Ignoring the constant terms, we can obtain the equivalent optimization problem:
\begin{equation}\label{eq:temp}
\min _{\mathbf{f}\in\mathbb{R}^n, \mathbf{g} \in \mathbb{R}^{m}} \epsilon \sum_{i, j=1}^{n} \exp \left(\frac{\mathbf{f}_i+\mathbf{g}_j-\mathbf{D}_{i j}}{\epsilon}\right)+\kappa\left\langle e^{-\mathbf{f} / \kappa}, \mathbf{a}\right\rangle+\kappa\left\langle e^{-\mathbf{g} / \kappa}, \mathbf{b}\right\rangle .
\end{equation}

According to the first-order condition, the minimizer's gradient of \eqref{eq:temp} should be zero.
As such, fixing $\mathbf{g}^\ell$, the updated $\mathbf{f}^{\ell+1}$ ought to satisfy:
\begin{equation}
    \exp \left(\frac{\mathbf{f}_i^{\ell+1}}{\epsilon}\right) \sum_{j=1}^{n} \exp \left(\frac{\mathbf{g}_j^{\ell}-\mathbf{D}_{i j}}{\epsilon}\right)=\exp \left(-\frac{\mathbf{f}_i^{\ell+1}}{\kappa}\right) \mathbf{a}_{i},
\end{equation}
We further multiply both sides by $\exp(\mathbf{f}_i^\ell/\epsilon)$:
\begin{equation}\label{eq:uot1}
    \exp \left(\frac{\mathbf{f}_i^{\ell+1}}{\epsilon}\right) \mathbf{a}_{i}^{\prime}=\exp \left(\frac{\mathbf{f}_i^{\ell}}{\epsilon}\right) \exp \left(-\frac{\mathbf{f}_i^{\ell+1}}{\kappa}\right) \mathbf{a}_{i}
\end{equation}
where $\mathbf{a}^{\prime}:=\bpi \mathbf{1}_n$ with $\bpi_{ij} := \exp(\mathbf{f}^\ell_i+\mathbf{g}^\ell_j-\mathbf{D}_{ij})$.
Similarly, fixing $\mathbf{f}$ we have $\mathbf{g}^{\ell+1}$ as:
\begin{equation}\label{eq:uot2}
    \exp \left(\frac{\mathbf{g}_j^{\ell+1}}{\epsilon}\right) \mathbf{b}_{j}^{\prime}=\exp \left(\frac{\mathbf{g}_j^{\ell}}{\epsilon}\right) \exp \left(-\frac{\mathbf{g}_j^{\ell+1}}{\kappa}\right) \mathbf{b}_{j}
\end{equation}
where $\mathbf{b}^{\prime}:=\bpi^{\mathrm{T}} \mathbf{1}_m$.
\eqref{eq:uot1} and \eqref{eq:uot2} construct the critical iteration steps of the generalized Sinkhorn algorithm~\citep{uot}, which we formulate in Algorithm~\ref{alg:uotsinkhorn}.
The transport matrix $\bpi^{\epsilon,\kappa}$ further induces the \emph{generalized Sinkhorn discrepancy} $\mathbb{W}^{\epsilon, \kappa}$ in Definition~\ref{def:uotapp}.
As $\mathbf{D}$ is differentiable with respect to $\alpha$ and $\beta$, it is feasible to minimize $\mathbb{W}^{\epsilon, \kappa}$ by adjusting the generation process of $\alpha$ and $\beta$, \ie the representation mapping in Definition~\ref{def:mapapp}, with gradient-based optimizers.

\subsection{Optimization of Entire Space Counterfactual Regression}
\begin{algorithm}[tb]
\caption{ESCFR Algorithm}\label{alg:escfr}
\textbf{Input}: covariates of treated units $\left\{\mathbf{x}_i\right\}_{i=1}^n$ and untreated units $\left\{\mathbf{x}_j\right\}_{j=1}^m$;
factual outcomes $\left\{y_i\right\}_{i=1}^n$ and $\left\{y_j\right\}_{j=1}^m$;
representation mapping $\psi$; outcome mapping $\phi$.\\
\textbf{Parameter}: $\lambda$: strength of optimal transport; $\epsilon$: strength of entropic regularizer; $\kappa$: strength of RMPR; $\gamma$: strength of PFOR; $\ell_\mathrm{max}$: max iterations \\
\textbf{Output}: $\mathcal{L}_{\mathrm{ESCFR}}^{\epsilon, \kappa, \gamma, \lambda}$: the learning objective of ESCFR.
\begin{algorithmic}[1] 
    \State
    $\left\{\mathbf{r}_i\right\}_{i=1}^n \gets \left\{\psi(\mathbf{x}_i)\right\}_{i=1}^n,\quad \quad
    \left\{\mathbf{r}_j\right\}_{j=1}^m \gets  \left\{\psi(\mathbf{x}_j)\right\}_{j=1}^m$.
    \State
    $\left\{\hat{y}_i\right\}_{i=1}^n \gets \left\{\phi(\mathbf{r}_i,1)\right\}_{i=1}^n,\quad
    \left\{\hat{y}_j\right\}_{j=1}^m \gets  \left\{\phi(\mathbf{r}_j,0)\right\}_{j=1}^m$.
    \State
    $\left\{\tilde{y}_i\right\}_{i=1}^n \gets \left\{\phi(\mathbf{r}_i,0)\right\}_{i=1}^n,\quad
    \left\{\tilde{y}_j\right\}_{j=1}^m \gets  \left\{\phi(\mathbf{r}_j,1)\right\}_{j=1}^m$.
    \State
    $\mathbf{D}_{ij}^\gamma \gets \Vert\mathbf{x}_i-\mathbf{x}_j\Vert_2^2 + \gamma \cdot \Vert y_i-\tilde{y}_j\Vert_2^2 + \gamma \cdot \Vert y_j-\tilde{y}_i\Vert_2^2$.
    \State
    $\mathbf{D}^\gamma_\mathrm{stop} \gets \mathrm{stop gradient}(\mathbf{D}^\gamma)$.
    \State
    $\bpi^{\epsilon,\kappa,\gamma} \gets \mathrm{Algorithm2}\left(\alpha=\left\{\mathbf{r}_i\right\}_{i=1}^n, \beta=\left\{\mathbf{r}_j\right\}_{j=1}^m, \mathbf{D}=\mathbf{D}^\gamma_\mathrm{stop}\right)$.
    \State
    $\mathcal{L}_{\mathrm{F}}(\psi, \phi)\gets
    \frac{1}{n}\sum_{i=1}^n \Vert\hat{y}_i-y_{i}\Vert_2^2+
    \frac{1}{m}\sum_{j=1}^m \Vert\hat{y}_j-y_{j}\Vert_2^2$.
    \State
    $\mathcal{L}_{\mathrm{D}}^{\epsilon, \kappa, \gamma}(\psi) \gets \left<\mathbf{D}^\gamma, \bpi^{\epsilon,\kappa,\gamma}\right>$.
    \State
    $\mathcal{L}_{\mathrm{ESCFR}}^{\epsilon, \kappa, \gamma, \lambda} \gets \mathcal{L}_{\mathrm{F}}(\psi, \phi)+\lambda \cdot \mathcal{L}_{\mathrm{D}}^{\epsilon, \kappa, \gamma}(\psi)$.
\end{algorithmic}
\end{algorithm}
Algorithm~\ref{alg:escfr} shows how to calculate the learning objective at a mini-batch level.
Specifically, we first calculate the factual outcome estimates (step 2), counterfactual outcome estimates (step 3), and the unit-wise distance matrix with PFOR (step 4).
Afterwards, fix the gradient of the distance matrix (step 5) and calculate the transport matrix with Algorithm~\ref{alg:uotsinkhorn} (step 6).
Finally, calculate the factual outcome estimation error (step 7) and distribution discrepancy (step 8), and aggregate them to acquire the learning objective of ESCFR (step 9). 
According to Section~\ref{sec:uotapp}, the learning objective is differentiable to $\psi$ and $\phi$ and thus can be optimized end-to-end with stochastic gradient methods.

\subsection{Complexity Analysis}
\begin{table*}[]
\caption{Running time (mean+std) in seconds of Algorithm 1-2 with 100 runs.}\label{tab:time}
\resizebox{\linewidth}{!}{
\setlength{\tabcolsep}{10pt}
\begin{tabular}{lllllll}
\toprule
Parameter    & $n=32$ & $n=64$ & $n=128$ & $n=256$ & $n=512$ & $n=1024$ \\
Algorithm1    & 0.0266$\pm$0.0102 & 0.0241$\pm$0.0075 & 0.0326$\pm$0.0088 & 0.0499$\pm$0.0099 & 0.0725$\pm$0.0128 & 0.1430$\pm$0.0259\\
Algorithm2    & 0.0050$\pm$0.0004 & 0.0051$\pm$0.0001 &  0.0065$\pm$0.0002 & 0.0104$\pm$0.0005 & 0.0138$\pm$0.0008 & 0.0256$\pm$0.0007\\ \midrule
Parameter    & $\epsilon=0.1$ & $\epsilon=0.5$ & $\epsilon=1.0$ & $\epsilon=5.0$ & $\epsilon=10.0$ & $\epsilon=100.0$ \\
Algorithm1    & 0.1683$\pm$0.0038 & 0.1207$\pm$0.0102 & 0.0699$\pm$0.0095 & 0.0153$\pm$0.0013 & 0.0097$\pm$0.0009 & 0.0072$\pm$0.0009\\
Algorithm2    & 0.0166$\pm$0.0019 & 0.0068$\pm$0.0010 &  0.0052$\pm$0.0011 & 0.0047$\pm$0.0010 & 0.0045$\pm$0.0011 & 0.0043$\pm$0.0009\\\midrule
Parameter    & $\kappa=0.1$ & $\kappa=0.5$ & $\kappa=1.0$ & $\kappa=5.0$ & $\kappa=10.0$ & $\kappa=100.0$ \\
Algorithm2 & 0.0050$\pm$0.0011 & 0.0059$\pm$0.0008 & 0.0060$\pm$0.0011 & 0.0112$\pm$0.0014 & 0.0162$\pm$0.0016 & 0.1039$\pm$0.0033\\
\bottomrule
\end{tabular}
}
\end{table*}
One primary concern would be the overall complexity of solving discrete optimal transport problems. 
Exact algorithms, \eg the interior-point method and network-simplex method, suffer from a high computational cost of $\mathcal{O}(n^3\log n)$~\citep{exact3}.
An entropic regularizer is thus introduced in \eqref{eq:eot}, making the problem solvable by the Sinkhorn algorithm~\citep{sinkhorn} in Algorithm 1.
The complexity was shown to be $\mathcal{O}(n^2/\epsilon^3)$ by \cite{sink1} in terms of the absolute error of the mass-preservation constraints. 
~\cite{sink2} improved it to $\mathcal{O}(n^2/\epsilon^2)$, which can be further accelerated with greedy algorithm by \cite{sink3}. 
Several recent explorations~\citep{sink4,sink5} have also attempted to further reduce the complexity to $\mathcal{O}(n^2/\epsilon)$.

Entropic regularization trick is still applicable to speed up the solution of the unbalanced optimal transport problem in RMPR, represented by the Sinkhorn-like algorithm
in Algorithm~\ref{alg:uotsinkhorn}. 
\cite{usink1} further proved that the complexity of Algorithm~\ref{alg:uotsinkhorn} is $\tilde{\mathcal{O}}(n^2/\epsilon)$.

Table~\ref{tab:time} reports the practical running time at the commonly-used batch settings. 
In general, the computational cost of optimal transmission is not a concern at the mini-batch level.
Notice that enlarging $\epsilon$ speeds up the computation while making the resulting transfer matrix biased, hindering the transportation performance, as per Figure~\ref{fig:param}. In addition, a large relaxation parameter $\kappa$ makes the computed results closer to those by Sinkhorn algorithm yet significantly contributes to more iterations, which is discussed and mitigated by ~\cite{usink2}.

\section{Reproduction Details}\label{apdx_c}
\subsection{Datasets}
We conduct experiments on two semi-synthetic benchmarks to validate our models. 
For the IHDP\footnote{It can be downloaded from https://www.fredjo.com/} benchmark, we report the results over 10 simulation realizations following~\cite{site}. However, the limited size (747 observations and 25 covariates) makes the results highly volatile.
As such, we mainly validate the models on the ACIC benchmark, which was released by the ACIC-2016 competition\footnote{It can be downloaded from https://jenniferhill7.wixsite.com/acic-2016/competition}.

All datasets are randomly shuffled and partitioned in a 0.7:0.15:0.15 ratio for training, validation, and test, where we maintain the same ratio of treated units in all three splits to avoid numerical unreliability in the validation and test phases.
We find that these datasets are overly easy to fit by the model because they are semi-synthetic.
To increase the distinguishability of the results, we omit preprocessing strategies, such as min-max scaling, to increase the difficulty of the learning task.

\subsection{Baselines}
The collection of baselines involves statistical estimators~\citep{kunzel2019metalearners,cfr}, matching estimators~\citep{psm,dmlforest,knn} and representation-based estimators~\citep{bnn,cfr}.
We implement these baselines based on Pytorch for neural network models, Sklearn for statistical models, and EconML for tree and forest models.

\section{Additional Discussions}\label{sec:app_discuss}
\subsection{Additional Discussion for Stochastic Optimal Transport}\label{sec:appendix_sot}
According to Theorem~\ref{thm:bound}, one critical hyperparameter for CFR-WASS and ESCFR is the batch size, which directly affects the variance of stochastic optimal transport in Section~\ref{sec:sot} and thus the performance of both methods. 
As such, it is necessary to verify whether ESCFR outperforms CFR for different batch sizes. We conduct extensive experiments and summarize the results in Table~\ref{tab:bs}.
Interesting observations are noted:
\begin{itemize}[leftmargin=*]
    \item Increasing batch size in a wide range improves the performance of CFR-WASS and ESCFR. 
    For example, The PEHE of CFR-WASS decreases from 3.114 at $b=32$ to 2.932 at $b=128$, and the PEHE of ESCFR exhibits a similar pattern.
    The performance gain is 
    attributed to the decreased variance in \eqref{eq:mbw}, which backs up Theorem~\ref{thm:bound}.
    \item By finetuning batch size, we can easily exceed the performance we report in Table~\ref{tab:result}. However, we did not finetune it as the PEHE is invisible during our hyper-parameter tuning process\footnote{Most of the experiments in Table~\ref{tab:result} were performed with a fixed batch size $b=32$, which is selected by the factual estimation performance of TARNet.}.
    \item The performance drop given overly large batch sizes comes from the sub-optimal backbone (TARNet) performance. Due to the limited training samples, \eg 4.8k * 70\% units for ACIC and 0.7k * 70\% units for IHDP, a large batch size might impede the optimizer from escaping saddle points~\citep{saddle} and sharp minima~\citep{minima}, thus 
    deteriorating the quality of factual outcome estimation.
\end{itemize}

\subsection{Additional Discussion for Relaxed Mass-Preserving Regularizer}\label{sec:appendix_rmpr}
\begin{table*}[]
\caption{Out-of-sample PEHE of ESCFR and important baselines with different batch sizes $b$.}\label{tab:bs}
\resizebox{\linewidth}{!}{
\setlength{\tabcolsep}{10pt}
\begin{tabular}{lllllll}
\toprule
Model & $b=32$ & $b=64$ & $b=96$ & $b=128$ & $b=196$ & $b=256$ \\
\midrule
TARNet & 3.3293$\pm$0.1853 & 3.2054$\pm$0.2676 & 3.0869$\pm$0.2812 & 2.9262$\pm$0.3160 & 3.4619$\pm$0.6652 & 3.6309$\pm$0.2026\\
CFR-WASS & 3.1143$\pm$0.4578 & 3.0819$\pm$0.3407 & 2.9998$\pm$0.1017 & 2.9326$\pm$0.4142 & 4.0740$\pm$1.4127 & 3.4675$\pm$0.1552\\
ESCFR & 2.3736$\pm$0.1621 & 2.3082$\pm$0.4334 & 2.9719$\pm$0.2889 & 2.3125$\pm$0.1836 & 2.0373$\pm$0.1538 & 2.2777$\pm$0.4230\\
\bottomrule
\end{tabular}
}
\end{table*}
 
Existing methods~\citep{cfr,bnn,site} suffer from the mini-batch sampling effect (MSE) issue, as indicated by the two bad cases in Figure~\ref{fig:match}.
RMPR mitigates the MSE issue by relaxing the mass-preserving constraint, the performance of which is affected by two critical hyperparameters, \ie the batch size $b$ and the strength of mass-preserving constraint $\kappa$. 
On top of the ablation studies, it is necessary to explore the performance of ESCFR at different settings of $b$ and $\kappa$, to investigate 1) how RMPR works; 2) the limitation and bottleneck of RMPR; 3) the robustness of RMPR to hyperparameter setting.
The results are 
presented in Figure~\ref{fig:batch}, and the observations are summarized as follows.
\begin{itemize}[leftmargin=*]
    \item The optimal value of $\kappa$ increases with the increase of batch size. For example, the optimal $\kappa$ is 1.0 at $b=32$, and 5.0 at $b=128$. This observation partially verifies how RMPR works as described in Section~\ref{sec:uot}. Specifically, at small batch sizes where sampling outliers dominate the sampled batches, a small $\kappa$ effectively relaxes the mass-preserving constraint and avoids the damage of mini-batch outliers, thus improving the performance 
    effectively and robustly. At large batch sizes, the noise of sampling outliers is reduced, and it is reasonable to increase $\kappa$ to match more units and obtain more accurate wasserstein distance estimates. 
    \item Even with large batch sizes, oversized $\kappa$, \eg $\kappa\geq10$ does not perform well. Although the effect of sampling outliers is reduced, some patterns such as outcome imbalance are present for all batch sizes, resulting in false matching given large mass-preserving constraint strength $\kappa$, which might be a primary bottleneck of RMPR.
    \item Hyper-parameter tuning is not necessarily the reason why ESCFR works well, since all ESCFR implementations outperform the strongest baseline CFR-WASS ( $\kappa=\infty$) on all batch sizes, most of which are statistically significant. This can be further supported by our extensive ablation study in Section~\ref{sec:ablation} and parameter study in Section~\ref{sec:parameter}.
\end{itemize}

In summary, it is necessary to relax the mass-preserving constraint under all settings of batch size, which strongly verifies the effectiveness of RMPR in Section~\ref{sec:uot}.
\begin{figure*}
    \centering
    \includegraphics[width=\linewidth, trim=10 30 10 0]{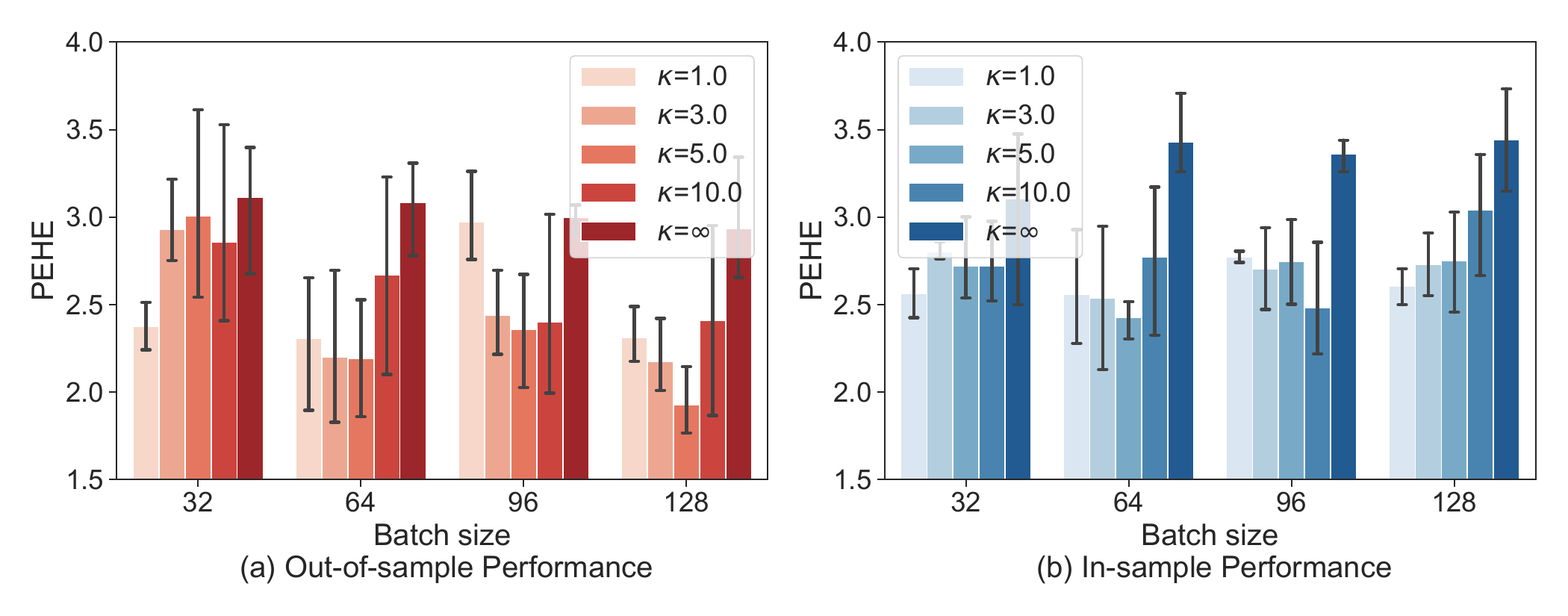}
    \caption{PEHE of ESCFR and CFR-WASS ($\kappa=\infty$) under different batch size.}\label{fig:batch}
\end{figure*}

\subsection{Additional Discussion for Proximal Factual Outcome Regularizer}\label{sec:appendix_pfor}
Existing representation-based methods block the backdoor path $X\rightarrow T $ by balancing the distribution of the observed covariates in a latent space. 
Given the unconfoundedness assumption~\ref{asm:unconfoundedness}, this approach effectively handles the treatment selection bias. 
However, Assumption~\ref{asm:unconfoundedness} is usually violated in practice, which invalidates this approach as the backdoor path from the unobserved confounder $X^\prime$ to $T$ is not blocked.

According to the designed causal graph in Figure~\ref{fig:causalgraph:b}, all factors associated with outcome $Y$ include the observed confounders $X$, treatment $T$, and unobserved confounders $X^\prime$.
Therefore, it is reasonable to derive that given balanced $X$ and identical $T$, the only variable reflecting the variation of $X^\prime$ is the outcome $Y$.
As such, inspired by the joint distribution transport 
technique~\citep[see][]{jdot}, PFOR calibrates the unit-wise distance $\mathbf{D}$ with the potential outcomes in \eqref{eq:pfor}.
The underlying regularization is: units with similar (observed and unobserved) confounders should have similar potential outcomes.
Equivalently, for a pair of units with similar observed covariates, \ie $\Vert r_i - r_j\Vert^2\approx 0$, if their potential outcomes under the same treatment $t=\{0,1\}$ differ significantly, \ie $\Vert y_i^{t}-y_j^{t}\Vert>>0$, their unobserved confounders should also differ significantly.
As such, it is reasonable to utilize the difference of outcomes to calibrate the unobserved confounding effect.
\begin{asmp}\label{asm:mono} (Monotonicity). For all observed covariates $X=x$ in the population of interest, let $T=t$ and $X^\prime=x^\prime$ be the treatment assignment and unobserved confounders, respectively, we have $\E[Y\mid X=x,X^\prime=x^\prime, T=t]$ is monotonically increasing or decreasing with respect to $x^\prime$.
\end{asmp}

\paragraph{Advantages.}The advantages of PFOR can be further interpreted as follows.
\begin{itemize}[leftmargin=*]
    \item From a statistical perspective, PFOR encourages units with similar outcomes to share similar representations. 
    It is a valid prior that inspires many learning algorithms, \eg K-nearest neighbors and gaussian process~\citep[see][]{gpml}. 
    As an effective statistical regularizer, PFOR also works in the absence of unobserved confounders, especially on small data sets.
    
    \item From a domain adaptation perspective, vanilla Sinkhorn aligns the distributions $\mathbb{P}^{\mathrm{T=1}}_\psi(r)$ and $\mathbb{P}^{\mathrm{T=0}}_\psi(r)$, where $r$ is the learned representations in Definition~\ref{def:distri}. PFOR further aligns the transition probabilities $\mathbb{P}^\mathrm{T=1}(Y(T=t)\mid r)$ and $\mathbb{P}^\mathrm{T=0}(Y(T=t)\mid r)$ for $t=0,1$. The discrepancy between transition probabilities can be attributed to unobserved confounders that can be viewed as parameters of the transition probabilities~\citep{jdot}. As such, it is feasible to align the unobserved confounders by aligning the transition probabilities.
\end{itemize}

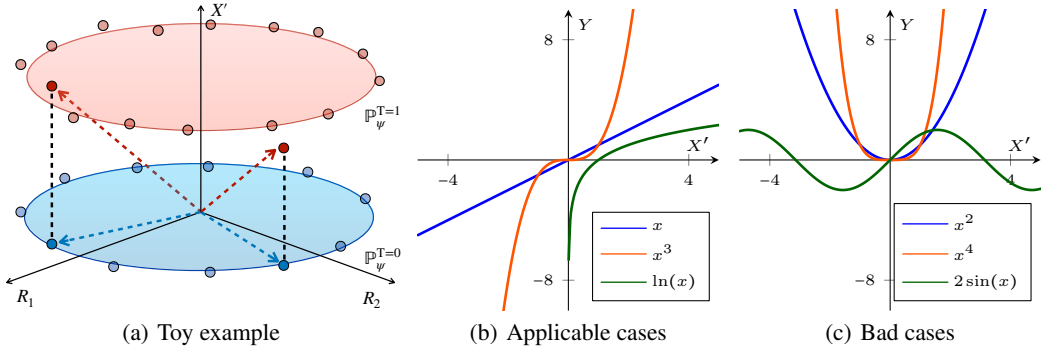
\begin{figure}
    \centering
    \subfigure[Toy example]{
    \includegraphics[width=0.38\textwidth]{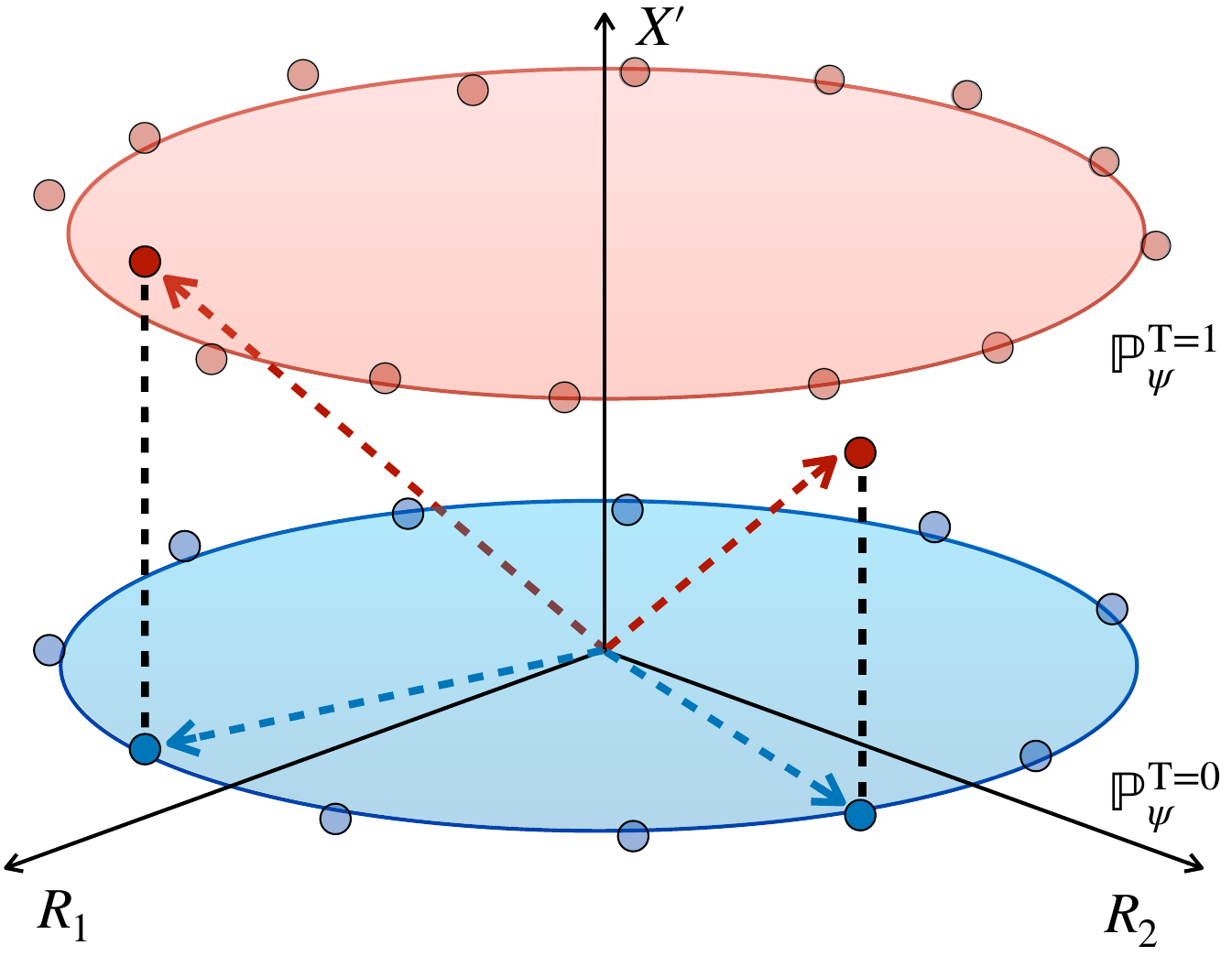}}
    \subfigure[Applicable cases]{
    \begin{tikzpicture}[line cap=round,line join=round,>=triangle 45,x=1cm,y=0.5cm]
    \begin{axis}[
        x=0.4cm,y=0.2cm,
        axis lines=middle,
        ymajorgrids=false, xmajorgrids=false,
        xmin=-5, xmax=5, ymin=-10, ymax=10,
        xtick={-4,0,...,4}, ytick={-8,0,...,8},
        xlabel={$X^\prime$},
        ylabel={$Y$},
        label style={font=\tiny},
        tick label style={font=\tiny},
        legend style={font=\tiny},
        legend cell align=left,
        legend pos=south east
        ]

        \draw[line width=1pt,color=qqqqff,smooth,samples=100,domain=-5:5] plot(\x,{(\x)^(1)});
        \addlegendimage{line width=0.5pt,color=qqqqff}
        \addlegendentry{$x$}
        \draw[line width=1pt,color=ffvvqq,smooth,samples=100,domain=-5:5] plot(\x,{(\x)^(3)});
        \addlegendimage{line width=0.5pt,color=ffvvqq}
        \addlegendentry{$x^3$}
        \draw[line width=1pt,color=qqwuqq,smooth,samples=100,domain=0.01:5] plot(\x,{log2(\x)});
        \addlegendimage{line width=0.5pt,color=qqwuqq}
        \addlegendentry{$\ln(x)$}
 
    \end{axis}
    \end{tikzpicture}
    }
    \subfigure[Bad cases]{
    \begin{tikzpicture}[line cap=round,line join=round,>=triangle 45,x=1cm,y=0.5cm]
    \begin{axis}[
        x=0.4cm,y=0.2cm,
        axis lines=middle,
        ymajorgrids=false, xmajorgrids=false,
        xmin=-5, xmax=5, ymin=-10, ymax=10,
        xtick={-4,0,...,4}, ytick={-8,0,...,8},
        xlabel={$X^\prime$},
        ylabel={$Y$},
        label style={font=\tiny},
        tick label style={font=\tiny},
        legend style={font=\tiny},
        legend cell align=left,
        legend pos=south east
        ]

        \draw[line width=1pt,color=qqqqff,smooth,samples=100,domain=-5:5] plot(\x,{(\x)^(2)});
        \addlegendimage{line width=0.5pt,color=qqqqff}
        \addlegendentry{$x^2$}
        \draw[line width=1pt,color=ffvvqq,smooth,samples=100,domain=-5:5] plot(\x,{(\x)^(4)});
        \addlegendimage{line width=0.5pt,color=ffvvqq}
        \addlegendentry{$x^4$}
        \draw[line width=1pt,color=qqwuqq,smooth,samples=100,domain=-5:5] plot(\x,{2*sin(deg(\x))});
        \addlegendimage{line width=0.5pt,color=qqwuqq}
        \addlegendentry{$2\sin(x)$}
 
    \end{axis}
    \end{tikzpicture}
    }
\caption{A diagram showing how PFOR works and its limitations. (a) A toy example of PFOR, where $R$ and $X^\prime$ indicate the balanced representations and an unobserved confounder, respectively; scatters indicate the empirical distribution of units in the treated and control groups; for solid scatters with balanced $R$, the colored dashed line indicates the ground truth outcome $Y=\sqrt{R_1^2+R_2^2+X^{\prime 2}}$ in each group, the black dashed line measures the difference of unobserved $X^\prime$. 
(b) Cases that satisfy Assumption~\ref{asm:mono}, where the the outcome $Y$ is monotone with unobserved $X^\prime$ given observed confounders in $R$.
(c) Cases that violate Assumption~\ref{asm:mono}, where the $Y$ is non-monotone with $X^\prime$.}
\label{fig:diagram}
\end{figure}

\paragraph{Toy example.} Let the ground truth $Y:=\sqrt{R_1^2+R_2^2+X^{\prime 2}}$ where $T$ is omitted as we only consider one group, $R_1$ and $R_2$ are the representations of observed confounders that have been aligned with Sinkhorn algorithm. Let the unobserved $X^\prime=0$ for controlled units and $X^\prime=1$ for treated units, which makes $X^\prime$ an unobserved confounder as it is related to Y and different between groups. As shown in Figure~\ref{fig:diagram}(a), given balanced $R_1$ and $R_2$, the variation of $Y$ reveals that of $X^\prime$. As such, it is reasonable to employ $Y$ to calibrate the unit-wise distance $\mathbf{D}$ that ignores $X^\prime$.
\paragraph{Synthetic labels.}
PFOR remains effective for semi-synthetic data, where the outcomes are synthetic from the covariates and treatment assignments. One source of hidden confounders in such data is information loss from the raw data space to the representation space, where not all valuable information (\eg confounders) is extracted and preserved, in particular when the representation mapping $\psi$ is not invertible. 
Besides, this improvement could also come from the statistical regularization, encouraging units with similar outcomes to share similar representations, which is an effective prior according to the K-nearest neighboring methods and warrants further investigation in the context of treatment effect estimation. 

\paragraph{Limitations.} PFOR fails to handle confounders that add constant effects to all units. Specifically, for unobserved confounder $X^\prime$ and treatment assignment $t=0,1$, if $\E[Y\mid X,X^\prime=x_1, T=t] = \E[Y\mid X,X^\prime=x_2, T=t]$, PFOR fails to eliminate the confounding effect of $X^\prime$. Examples can be found in Figure~\ref{fig:diagram} (c). However, in real scenarios, it is rare that different values of $X^\prime$ only add a constant effect to the outcome~\citep[see][]{asmp1,asmp2,asmp3}, making PFOR still effective in a wide range of application scenarios. 

This limitation is formalized as the Assumption~\ref{asm:mono}, where the outcome should increase or decrease monotonically
with unobserved confounders given observed confounders and treatment assignment, as shown in Figure~\ref{fig:diagram} (b). Notably, it is a commonly used assumption in confounder analysis~\citep{asmp1,asmp2}. Besides, this assumption is often plausible, at least approximately, conditional on $T=t$~\citep{asmp2} . For example, it naturally holds for binary confounders; and 
generally holds in applications such as epidemiology~\citep{asmp3}. Finally, this assumption is only imposed on the hidden confounder $X^\prime$ following~\cite{asmp2}, which further weakens Assumption~\ref{asm:mono} significantly.

\paragraph{Further discussion.} PFOR is mainly built upon the assumption of the causal graph in Figure~\ref{fig:causalgraph:b}, where the roles of the adjustment variables and the noise variables are excluded. Actually, it is a standard setting in many existing work of causal inference, such as Figure 1.1 in~\cite{sauer2013review} and Figure 5 in~\cite{brookhart2006variable}. Nevertheless, we would like to further discuss the applicability of PROR when there are noise variables that are parents only of $Y$. We provide the following analysis.
\begin{itemize}[leftmargin=*]
    \item If these variables are both observable and predictive of $Y$, they are known as adjustment variables. Aligning these variables does not introduce additional bias and can lead to a reduction in the variance of estimated treatment effects~\cite{zheng2021sensitivity,garrido2021estimating}. Therefore, many studies~\cite{disentangle,mim} do not differentiate them from confounders and align them together with confounders. Therefore, this is not an issue with ESCFR as these variables can be adjusted along with confounders.
    \item If these variables are non-observable and predictive of $Y$, we can rely on the monotonicity assumption to adjust for them using PFOR alongside unobserved confounders. This method does not introduce any additional bias and can still reduce the variance in the estimated treatment effects~\cite{zheng2021sensitivity,garrido2021estimating}.
    \item If these variables are pure noise, \ie non-predictive of $Y$, we believe they will interference the calculation of PFOR. Nevertheless, we mildly argue that the effect of this noise is not catastrophic, since such independent noise is also present in X and does not impede the success of canonical OT in fundamental domains such as computer vision and neural language processing.
    \item Finally, we find it would be interesting to discuss the robustness of the OT discrepancy to the volume of noise (in both $X$ and $Y$), with the aim at devising a more robust OT discrepancy. There have been many attempts in this topic, including but not limited to UOT, Relaxed OT, semi-UOT, etc. These robust approaches could further mitigate the negative impact of noise and handle this piece of weakness. In future work, we will allocate some effort to exploring this interesting topic.
\end{itemize}

An important approach with unobserved confounders is the partial identification. Specifically, an estimand denoted as $\theta$ is partially identified if the observed data distribuion is compatible with multiple values of $\theta$. In causal inference, challenges like unobserved confoundings might prevent precise causal effect pinpointing. A weaker alternative is to obtain a range of possible causal effects, known as a "identified set", and reduce the size of the set using proper assumptions. For example, given certain assumptions, \eg monotone treatment selection, we can narrow the bound of treatment effect estimate. 

It is interesting to investigate the connection between PFOR from the partial identification view.
We note that the transport strategy derived by the canonical Kantorovich problem in \eqref{eq:kanto} is non-identifiable given the existence of (an) unobserved confounder $X^\prime$. 
That is, assuming that $X^\prime$ has multiple candidate values, there should be multiple corresponding transport strategies, which makes the ground-truth transport strategy non-identifiable.
Ideally, we can only identify a huge strategy set (by enumerating possible values of $X^\prime$). Nevertheless, under the monotonic assumption between $X^\prime$
and outcomes, we can calibrate the unit-wise distance with the outcome differences, to reduce the size of strategy set and achieve more accurate estimation among possible transport strategies, which largely share the intuition of partial identification methodology. The partial identification method in causality from an OT view warrants further investigation.